\documentclass{article}
\usepackage{amsmath, amssymb, amsthm, amsfonts}
\usepackage{lmodern}
\usepackage{arxiv}
\usepackage{physics}
\usepackage[T1]{fontenc}
\usepackage{graphicx}
\usepackage{xcolor}
\usepackage{wrapfig} %
\usepackage{subcaption}
\usepackage{authblk}

\newcommand\blfootnote[1]{%
  \begingroup
  \renewcommand\thefootnote{}%
  \footnotetext{#1}%
  \endgroup
}

\usepackage{amsmath,amsfonts,bm}

\def\eqref#1{equation~\ref{#1}}

\def\1{\mathbf{1}}

\DeclareMathAlphabet{\mathsfit}{\encodingdefault}{\sfdefault}{m}{sl}
\SetMathAlphabet{\mathsfit}{bold}{\encodingdefault}{\sfdefault}{bx}{n}

\newcommand{\E}{\mathbb{E}}

\def\w{{\mathbf{w}}}

\def\N{\mathcal{N}}

\def\I{\mathbf{I}}

\newtheorem{theorem}{Theorem}
\newtheorem{proposition}{Proposition}
\newtheorem{lemma}{Lemma}

\newtheorem{corollary}{Corollary} 
\newcommand{\T}{\mathcal{T}}
\newcommand{\bigO}{\mathcal{O}}

\renewcommand{\vec}{\text{vec}}
\newcommand{\diag}{\text{diag}}
\newcommand{\kron}{\otimes}

\usepackage[utf8]{inputenc} %
\usepackage{enumitem}
\usepackage[T1]{fontenc}    %
\usepackage{hyperref}       %
\usepackage{url}            %
\usepackage{booktabs}       %
\usepackage{amsfonts}       %
\usepackage{nicefrac}       %
\usepackage{microtype}      %
\usepackage{xcolor}         %
\usepackage{cleveref}

\title{Compute Efficiency and Serial Runtime Tradeoffs \\for Stochastic Momentum Methods}

\author[*,1,2,$\dagger$]{Depen Morwani}
\author[*,1,2]{Alexandru Meterez}
\author[*,1,2]{Pranav Ajit Nair}
\author[1,2]{Sham Kakade}

\affil[1]{Harvard University}
\affil[2]{Kempner Institute at Harvard University}
\begin{document}

\maketitle

\begin{abstract}
Stochastic momentum methods such as heavy ball (HB), Nesterov momentum, and
variants of Accelerated SGD (ASGD)~\citep{kidambi2018insufficiency} are widely
used in modern training, but their stochastic benefits depend on two distinct
quantities: serial runtime, the number of iterations needed to reach a target
accuracy, and compute efficiency (CE), the inverse total gradient-query or FLOP
cost. Larger batches reduce serial runtime without hurting CE only when the
contraction gap grows linearly with batch size.

We study stochastic HB and ASGD for consistent linear regression with Gaussian
covariates and prove finite-dimensional, discrete-time \emph{lower} bounds on their batch-size tradeoffs. Our first result shows that HB does \emph{not} improve the CE frontier over SGD for arbitrary spectra; rather, it preserves SGD-level CE over a larger batch-size window, allowing larger batches to reduce serial runtime
until HB reaches its deterministic accelerated scale. This window can be a
factor \(\sqrt{\kappa}\) larger than the SGD critical batch size. For ASGD, the
picture is more spectrum-dependent: for rapidly decaying power-law spectra,
ASGD improves small-batch CE over HB/SGD, but as batch size grows it trades this
CE advantage for improved serial runtime. Synthetic linear-regression
experiments verify these qualitative regimes, including near-overlap of ASGD
and HB for slowly decaying spectra and the predicted CE--serial tradeoff for
rapidly decaying spectra.
\end{abstract}

\section{Introduction}
\blfootnote{\hspace{-6mm}${}^*$: Equal contribution. \\ $\dagger$: Work done while at Harvard.}
Training deep neural networks is typically done via a first order optimization method, such as gradient descent (GD), or in the case of large scale datasets and models that are memory prohibitive through minibatch Stochastic Gradient Descent (SGD)~\citep{robbins1951stochastic}, with the latter method taking gradient steps on a random subsample of the population. While most open source large scale models~\citep{grattafiori2024llama,team2026kimi,guo2025deepseek,yang2025qwen3} are trained with preconditioned SGD, typically Adam variants~\citep{kingma2014adam, loshchilov2017decoupled, vyas2024soap}, or the more recent approximate second order methods~\citep{martens2015optimizing,gupta2018shampoo,jordan2024muon}, they all generally bake in a version of momentum. 
Assuming a noiseless quadratic model, under exact gradients classical momentum methods such as heavy ball~\citep{polyak1964some} and Nesterov's accelerated gradient~\citep{nesterov1983method,nesterov2013introductory} improve over the GD rate of $\bigO(\kappa)$ to $\bigO(\sqrt{\kappa})$. However, this acceleration does not extend to the stochastic regime~\citep{kidambi2018insufficiency}, limiting their practicality for current training regimes. 
\begin{figure}[!htp]
    \centering
    \includegraphics[width=1.0\linewidth]{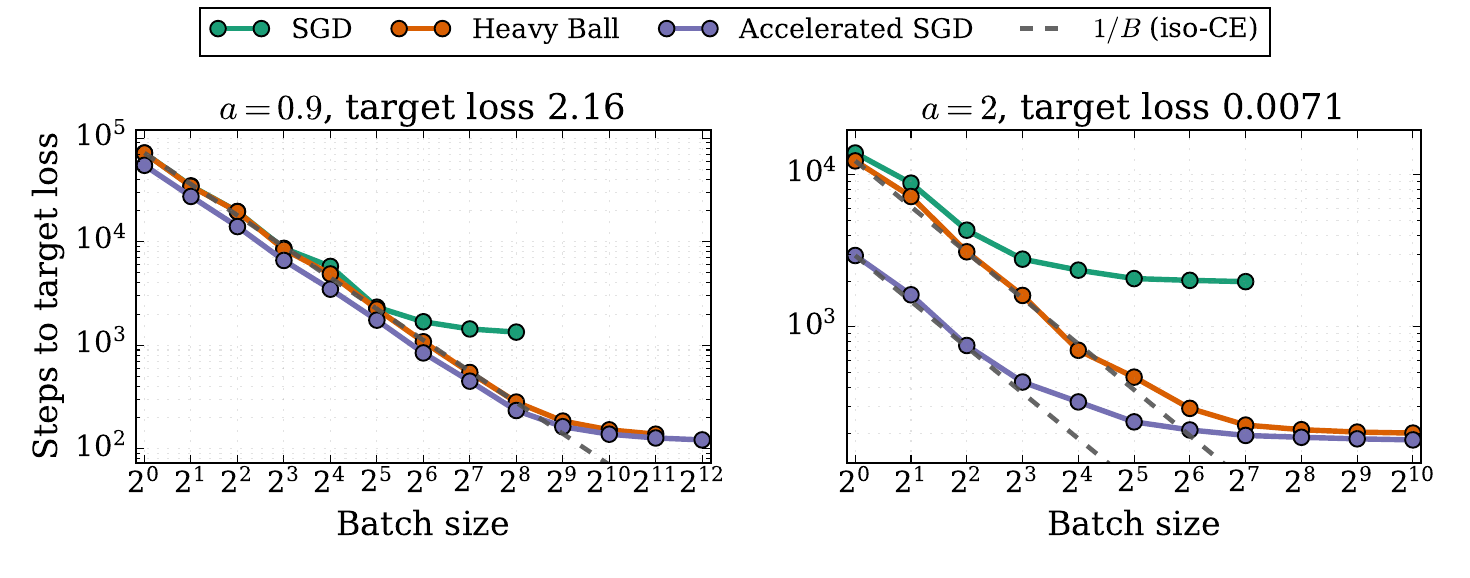}
    \caption{Batch-size tradeoffs for stochastic momentum methods in synthetic linear regression with power-law spectra. Plots show the number of serial steps needed to reach a fixed target loss as the batch size varies, after tuning hyperparameters for each method. (Left) For slow decaying power law spectra, HB and ASGD do not yield better CE, but they decrease the serial runtime. (Right) For fast decaying spectra, ASGD has better CE than HB, but has a shorter linear scaling window before it begins trading off CE for serial runtime. Experimental details are provided in Section~\ref{sec:experiments}.}
    \label{fig:synth_plot}
\end{figure}
Empirically, several works have found that momentum primarily accelerates in the large batch size regime, while in the small batch case its effects vanish~\citep{shallue2019measuring,fu2023and,kunstner2023noise,wang2023marginal,zhang2019algorithmic}. Most recently,~\citet{marek2025small} conducted a large scale empirical study in large language model (LLM) pretraining, and have found that at small batch sizes momentum does not bring any benefits over SGD. These works indicate a subtle interplay between compute efficiency and serial runtime in stochastic acceleration. Theoretically, there has been partial progress in understanding stochastic momentum at batch size $1$, through the lower bounds established by~\citet{kidambi2018insufficiency} and the asymptotic characterization of~\citet{lee2022trajectory, ferbach2025dimension}. We refer the reader to Section~\ref{sec:related_work} where we discuss these works further.
Based on the the theoretical understanding, new momentum schemes have been proposed combining multiple momentum buffers that provably achieve acceleration at batch size $1$~\citep{jain2018accelerating,vaswani2019fast,liu2018accelerating,gupta2024nesterov, ferbach2025dimension, pagliardini2024ademamix}, a family of methods we will refer to as Accelerated SGD (ASGD)~\citep{kidambi2018insufficiency,jain2018accelerating,ferbach2025dimension,lee2022trajectory}. While these methods achieve accelerated rates and thus better compute efficiency at batch size $1$, this regime is far from practical from a serial runtime point of view, especially in the current pretraining setups where the number of tokens can scale to several trillions. 

Thus, in large-scale pretraining, the key question is not only whether an optimizer improves sample complexity, but also whether it can maintain those gains at the large batch sizes needed to reduce serial training time. While keeping the data size fixed, one can linearly increase the batch size until the critical batch size (CBS)~\citep{zhang2024does,shallue2019measuring,merrill2025critical,meterez2025seesaw} while achieving a linear decrease in the serial cost. While the CBS for SGD has been extensively studied empirically and theoretically in linear regression, to the best of our knowledge there are no results on the CBS of momentum algorithms. In this paper we provide a nearly full characterization of the serial runtime tradeoffs by providing lower bounds for \emph{arbitrary} spectra for heavy ball, and power law spectra for ASGD.

\subsection{Contributions}
In the deterministic, full-batch regime, the classical story is clear:
on a \(\kappa\)-conditioned quadratic, GD has contraction gap of order
\(\kappa^{-1}\), while HB achieves the accelerated gap of order
\(\kappa^{-1/2}\). Equivalently, HB reduces the number of serial
iterations by a factor of order \(\sqrt{\kappa}\). Our goal is to
understand what remains of this acceleration in the stochastic
mini-batch regime, where we care about \emph{both} the
serial runtime and the compute efficiency. 

A useful way to compare stochastic momentum methods is to separate these two
quantities. For a given algorithm, let $s^*(H,B)$
denote the best-achievable spectral radius, or contraction factor, at batch size
\(B\) and with covariance spectrum \(H\) (suppressing dependence on
the algorithm). We define the corresponding contraction gap by
\[
\alpha(B):=1-s^*(H,B).
\]

\begin{table}[!ht]
\centering
\small
\begin{tabular}{lccc}
\toprule
Algorithm & Batch-size regime & Serial runtime & Compute efficiency \\
\midrule
SGD
& $1 \le B \lesssim \frac{\Tr(H)}{\lambda_{\max}}$
& $\frac{\Tr(H)}{B\lambda_{\min}}$
& $\frac{\lambda_{\min}}{\Tr(H)}$
\\
SGD
& $B \gtrsim \frac{\Tr(H)}{\lambda_{\max}}$
& $\kappa$
& $\frac{1}{B\kappa}$
\\
\midrule
HB
& $1 \le B \lesssim \frac{\Tr(H)}{\lambda_{\max}}\sqrt{\kappa}$
& $\frac{\Tr(H)}{B\lambda_{\min}}$
& $\frac{\lambda_{\min}}{\Tr(H)}$
\\
HB
& $B \gtrsim \frac{\Tr(H)}{\lambda_{\max}}\sqrt{\kappa}$
& $\sqrt{\kappa}$
& $\frac{1}{B\sqrt{\kappa}}$
\\
\bottomrule
\end{tabular}
\vspace{2mm}
\caption{
Serial-runtime and compute-efficiency implications of our lower bounds for SGD
and HB under arbitrary spectra. Here
$\kappa=\lambda_{\max}/\lambda_{\min}$, serial runtime scales as
$\alpha(B)^{-1}$, and compute efficiency scales as $\alpha(B)/B$, where
$\alpha(B)=1-s^*(H,B)$ is the best-achievable spectral gap. SGD and HB have the
same compute-efficiency frontier $\lambda_{\min}/\Tr(H)$, but HB preserves this
frontier over a larger batch-size window, up to
$B\asymp (\Tr(H)/\lambda_{\max})\sqrt{\kappa}$, thereby reducing serial runtime
from $\kappa$ to $\sqrt{\kappa}$ before saturating.
}
\vspace{-5mm}
\label{tab:sgd-hb-ce-serial}
\end{table}

In the full-batch limit, this gap satisfies $
\alpha_{\mathrm{GD}}(\infty)\asymp \kappa^{-1}$ and 
$\alpha_{\mathrm{HB}}(\infty)\asymp \kappa^{-1/2}$,
which is the classical deterministic acceleration of HB over GD.

The number of serial steps needed to reach a fixed target error scales as
\(1/\alpha(B)\), while the total gradient-query cost, or total FLOPs up to
constants, scales as \(B/\alpha(B)\). Equivalently, the compute efficiency
(CE) scales as
\[
\frac{\alpha(B)}{B}.
\]
Thus, higher CE means fewer total FLOPs to reach a target accuracy. If
\(\alpha(B)\) grows linearly with \(B\), increasing the batch size reduces
serial runtime without hurting CE. If \(\alpha(B)\) grows sublinearly with
\(B\), larger batches still improve serial runtime, but only by sacrificing CE.

Our first main result characterizes how this deterministic acceleration appears as the batch size is
varied, where we compare HB to SGD, and is summarized in Table~\ref{tab:sgd-hb-ce-serial}. The result holds for \emph{arbitrary} spectra.

\begin{center}
\textbf{HB improves the serial runtime over SGD, but it
does not improve the CE frontier.}
\end{center}

More precisely, we prove a lower bound showing that HB cannot improve over the
optimal CE scaling of SGD at any batch size. At small batch sizes, HB has the
same linear-scaling behavior as SGD. Above the SGD critical batch size, however,
HB can continue to convert larger batches into fewer serial steps while
preserving the same CE scale. This continues until HB reaches its deterministic
accelerated scale. Thus, HB's benefit is not a better best-achievable CE
frontier, but rather a larger batch-size window over which SGD-level CE can be
used to buy improved serial runtime.

In practice, several newer stochastic momentum methods are also being
considered. The majority of these schemes are paramterically
equivalent to ASGD, which is known to improve CE over SGD
at small batch sizes. Our second main result characterizes how this CE
improvement changes with batch size.

\begin{center}
\textbf{ASGD improves small-batch CE, but its CE-preserving linear-scaling
window is shorter.}
\end{center}

We show that ASGD's CE gains are concentrated in the small-batch regime. Unlike
the case for HB, the serial-runtime picture for ASGD is more dependent on the
full spectra. For rapidly decaying spectra, ASGD has an initial
CE-preserving linear-scaling regime. After this regime, its contraction gap
continues to improve with batch size, but only \emph{sublinearly}. Thus,
increasing the batch size still reduces ASGD's serial runtime, but it does so
by spending some of ASGD's small-batch CE advantage.

More concretely, for rapidly decaying spectra, the resulting picture is as follows. At small
batch sizes, ASGD has better CE than HB, consistent with known upper bounds at
\(B=1\). As \(B\) increases, ASGD moves along a CE--serial tradeoff curve:
serial runtime improves, but CE decreases relative to the small-batch optimum.
Eventually, ASGD reaches the deterministic accelerated serial-runtime scale.
Beyond this point, further increasing the batch size no longer improves serial
runtime and only decreases CE at the usual \(1/B\) rate. 

\begin{table}[!htp]
\centering
\small
\begin{tabular}{lccc}
\toprule
Algorithm & Batch-size regime & Serial runtime & Compute efficiency \\
\midrule
SGD
& $B \gtrsim 1$
& $d^a$
& $B^{-1}d^{-a}$
\\
\midrule
HB
& $1 \le B \lesssim d^{a/2}$
& $\frac{d^a}{B}$
& $d^{-a}$
\\
HB
& $B \gtrsim d^{a/2}$
& $d^{a/2}$
& $B^{-1}d^{-a/2}$
\\
\midrule
ASGD
& $B \asymp 1$
& $d^{\frac{a^2}{2a-1}}$
& $d^{-\frac{a^2}{2a-1}}$
\\
ASGD
& $1 \lesssim B \lesssim d^{1/2}$
& $d^{\frac{a^2}{2a-1}}B^{-\frac{a}{2a-1}}$
& $d^{-\frac{a^2}{2a-1}}B^{-\frac{a-1}{2a-1}}$
\\
ASGD
& $B \gtrsim d^{1/2}$
& $d^{a/2}$
& $B^{-1}d^{-a/2}$
\\
\bottomrule
\end{tabular}
\vspace{2mm}
\caption{
Serial-runtime and compute-efficiency implications of our lower bounds under
power-law spectra $\lambda_i\asymp i^{-a}$ with $a>1$. Serial runtime scales as
$\alpha(B)^{-1}$ and compute efficiency scales as $\alpha(B)/B$, where
$\alpha(B)=1-s^*(H,B)$ is the best-achievable spectral gap. For these spectra,
SGD has only a constant-size linear-scaling window. HB preserves SGD-level CE
over the larger window $B\lesssim d^{a/2}$, while ASGD improves small-batch CE
but loses CE as $B$ grows through the intermediate regime.
}
\vspace{-5mm}
\label{tab:batch-regimes}
\end{table}
Thus, ASGD improves
the CE--serial frontier most strongly at small batch sizes; for larger batch
sizes the advantage is more modest, because part of the small-batch CE gain has
already been converted into serial-time improvement. For more slowly decaying
spectra, ASGD and HB are comparable.

Figure~\ref{fig:synth_plot} illustrates these CE--serial tradeoffs in synthetic
linear regression experiments. The left panel shows the slowly decaying-spectrum
regime, where ASGD and HB are nearly overlapping while HB improves serial
runtime over SGD, at essentially the same CE as SGD. The right panel shows the rapidly decaying-spectrum regime predicted by our theory: ASGD improves small-batch CE over HB/SGD, and then spends this CE advantage to reduce serial runtime as the batch size grows. Table~\ref{tab:batch-regimes} summarizes the corresponding power-law scalings.

In summary, HB preserves SGD-level CE over a larger batch-size range and uses
this to improve serial runtime, but it does not improve the CE frontier itself.
ASGD improves the small-batch CE frontier, but its linear-scaling window is
shorter; beyond this window, larger batches buy serial-time improvements by
spending some of ASGD's CE advantage.

\section{Related Work}
\label{sec:related_work}
\paragraph{Acceleration with Noisy Gradients.} The suboptimality of GD~\citep{cauchy1847methode} for quadratic models has been well studied in literature~\citep{nesterov2013introductory}. In their seminal works,~\citet{polyak1964some} and ~\citet{nesterov1983method} have proposed different variants of momentum, namely heavy ball and Nesterov's accelerated gradient method, which improve the deterministic rate to $\bigO(\sqrt{\kappa})$. Several works have shown that for SGD with batch size $1$, both HB and NAG do not improve its compute efficiency~\citep{jain2018accelerating,kidambi2018insufficiency}. Extensions to these algorithms have been studied by introducing an extra momentum buffer~\citep{kidambi2018insufficiency,jain2018accelerating,gupta2024nesterov}, a family of algorithms commonly refered to as accelerated SGD~\citep{morwani2025connections,ferbach2025dimension}, which provably improve the contraction at batch size $1$. 

Closely related to our work is the work of~\citet{lee2022trajectory},
who analyze heavy-ball momentum for high-dimensional random least squares
in a proportional asymptotic regime. They consider mini-batch sizes
$\beta$ satisfying $\beta/n \to \zeta>0$ as the sample size $n$ and
feature dimension $d$ grow with $d/n$ fixed, and show that acceleration
appears only once the batch fraction crosses a spectrum-dependent
implicit conditioning ratio (ICR), a notion analogous to a critical batch
fraction. For upper bounds,~\citet{liu2018accelerating} established upper bounds at arbitrary batch size for an algorithm named MaSS, which is schematically similar to ASGD. 

\paragraph{ASGD Variants in Practice.} Several momentum schemes based on ASGD have been used in practice, in particular for large language model (LLM) pretraining.~\citet{morwani2025connections} have shown that the Schedule-Free algorithm~\citep{defazio2024road} can be rewritten as an ASGD equivalent. AdEMAMix~\citep{pagliardini2024ademamix} uses a similar scheme, based on a fast and slow momentum buffer updated with different EMA parameters. Interestingly, Lion~\citep{chen2023symbolic}, an algorithm discovered via genetic algorithms, also interpolates between the current gradient and momentum, before applying $\mathtt{sign}$ on the update. Several other methods have been proposed in the literature that aggregate over more than $1$ momentum buffer~\citep{lucas2018aggregated,ma2018quasi}. DANA~\citep{ferbach2025dimension,bordelon2026theory} is parametrically related to ASGD:
it uses a single momentum buffer together with a direct gradient path in
the parameter update, but, notably, chooses the parameters as a function of the model size and training time, rather than keeping them constant.

\paragraph{Nature of the bounds.}
Our results are finite-dimensional, discrete-time lower bounds. For HB, the
bound holds for arbitrary finite spectra; for ASGD, the stated comparison holds
under the two-sided power-law spectral assumption. In both cases, the bounds
hold uniformly over all stable parameter choices, including aggressive choices
near the edge of stability, and only hide universal constants. Although we
summarize the power-law consequences using large-\(d\) scaling notation, these
scalings come from finite-\(d\) inequalities rather than asymptotic limits. For heavy ball,~\citet{lee2022trajectory} provide a lower bound in a proportional limit, showing that there is no acceleration in the small batch size regime.

\paragraph{Divergence thresholds.} The techniques used in this work to derive the lower bounds also, implicitly, characterize the edge of stability~\citep{cohen2021gradient, cohen2022adaptive} conditions of the discussed algorithms. Recently,~\citet{andreyev2026momentum} have derived stability thresholds for heavy ball and Nesterov momentum, which we also implicitly recover through the proof of Theorem~\ref{thm:hb_lower_bound}.

\section{Background and Preliminaries}
\label{sec:main_results}
In this section, we first introduce the setup necessary for our main theoretical results.

\label{sec:setup_and_notation}
\paragraph{Definitions.} We denote as $f(x) \lesssim g(x)$ if there exists a constant $c>0$ such that $f(x) \leq c g(x)$ for any $x$ in the domain. Moreover, we denote as $f(x) \eqsim g(x)$ if $g(x) \lesssim f(x) \lesssim g(x)$. For $2$ matrices $A$ and $B$ we say that $A \preceq B$ if and only if $B - A \succeq 0$, where a matrix $B \succeq 0$ denotes that $B$ is positive semidefinite (PSD). We study stochastic optimization for a noiseless linear regression problem with Gaussian covariates i.e. a consistent linear system. 

We define our notion of compute to be the total number of gradient queries. For the covariance dynamics of the iterate, let $\T$ denote the linear operator such that
\[
\Sigma_{t+1}=\T(\Sigma_t).
\]
Following prior work on consistent linear systems~\citep{wu2022last,zou2021benign,meterez2025simplified}, the error contracts geometrically at a rate governed by the spectral radius of $\T$.

We train online, without repeating over the data. At each iteration, we sample from the population a minibatch of size $B$ of covariates $\{(x_i, y_i)\}_{i=1}^B$ such that:
\[
x_i \sim \mathcal N(0,H),
\qquad
y_i = \langle w^*, x_i\rangle,
\]
where \(w^* \in \mathbb R^d\) is the minimizer and $H \in \mathbb R^{d\times d}$, $H \succ 0$ is the population covariance matrix of the covariates. We denote the eigendecomposition of \(H\) by
\[
H = Q \Lambda Q^\top,
\qquad
\Lambda = \diag(\lambda_1,\dots,\lambda_d),
\qquad
\lambda_{\max}=\lambda_1 \ge \cdots \ge \lambda_d = \lambda_{\min} > 0
\]
We will use $\kappa = \frac{\lambda_{\max}}{\lambda_{\min}}$ to denote the conditioning number of $H$. Where stated, we will additionally specialize to a power-law spectrum  with exponent $a$ defined as:
\begin{align*}
    \lambda_i \eqsim i^{-a} &&
\lambda_{\max} \eqsim 1&&
\lambda_{\min} \eqsim d^{-a}
\end{align*}
Moreover, we denote the population risk as:
\[
\mathcal R(w)
:= \frac12 \, \E\big[(\langle w,x\rangle - y)^2\big]
\]
where the expectation is taken over the randomness induced by the stochastic gradient updates. We will recall a basic identity used repeatedly in the stochastic covariance analysis, namely the Gaussian fourth-moment formula. 
\begin{proposition}
    \label{prop:gaussian_forth}
    For any deterministic matrix $\Sigma$ and $x \sim \N(0, H)$ we have the following equation:
    \begin{equation}
    \E[xx^\top \Sigma xx^\top]
    =
    2H \Sigma H
    +
    \Tr(H \Sigma) H \preceq 3 \Tr(H \Sigma) H.
    \end{equation}
\end{proposition}
For a minibatch of size $B$ of covariantes with empirical covariance $\bar{X} = \frac1B \sum_{i=1}^B x_i x_i^\top$, Proposition~\ref{prop:gaussian_forth} generalizes as:
\begin{equation}
\label{eq:gaussian-fourth-b}
\E[\bar{X} \Sigma \bar{X}]
=
\left(1 + \frac1B \right)H \Sigma H
+
\frac1B \Tr(H \Sigma) H \preceq \left(1 + \frac2B\right) \Tr(H \Sigma) H
\end{equation}

\paragraph{Background.} We analyze the ASGD algorithm~\citep{jain2018accelerating,kidambi2018insufficiency,morwani2025connections}, with the following update rule:
\begin{align}
\label{eqn:asgd_update}
    \mu_t &= \beta \mu_{t-1} + (1-\beta)g_t \notag \\
    w_{t+1}&=w_t-\eta (\mu_t + \zeta g_t)
\end{align}
We denote as $\mu_t$ the momentum buffer at time $t$, the current iterate $w_t$ and the current minibatch gradient $g_t$. The hyperparameters of this algorithm are the learning rate $\eta$, the momentum EMA parameter $\beta$ and the weight placed on the current gradient $\zeta$. Note that by setting $\zeta = 0$, we recover the HB update:
\begin{align}
\label{eqn:hb_update}
    \mu_t &= \beta \mu_{t-1} + (1-\beta)g_t \notag \\
    w_{t+1}&=w_t - \eta \mu_t
\end{align}
Therefore, intuitively, since ASGD has an extra hyperparameter we can optimize over, it is possible that it can improve over the rates achieved by HB. 
\section{Compute Efficiency Lower Bounds}

We analyze these algorithms through the covariance dynamics of an augmented state vector $z_t$ dependent on the distance to the optimizer $w_t - w^\star$, and includes all variables needed to make the dynamics first-order linear in the covariance $\Sigma_t = \E [z_t z_t^\top]$. For each method, the error covariance satisfies a linear recursion
\[
\Sigma_{t+1} = \T(\Sigma_t)
\]
where $\T$ is a positive operator defined on the space of PSD matrices that depends on the data covariance $H$, batch size $B$ and the algorithm hyperparameters. Since the excess risk is a linear function of $\Sigma_t$~\citep{wu2022last,zou2021benign,wu2022power,meterez2025simplified}, the risk contraction is mainly governed by the \textit{spectral radius} $\rho(\T) = \max_i |t_i|$ where $t_i$ are the eigenvalues of $\T$. Thus, establishing lower bounds on $\rho(\T)$, or equivalently upper bounds on the \textit{spectral gap} $s = 1-\rho(\T)$ would directly imply lower bounds on the error rate for a consistent linear system. We leave the formal extension of our spectral bounds to risk bounds to future work.

We first state a general HB lower bound valid for arbitrary spectra, and then specialize it to power-law spectra to identify the critical batch-size transition. We next give an analogous bound for ASGD, showing that for power-law spectra the ASGD lower bound has a strictly better scaling than the HB lower bound in the small-batch regime and reaches the accelerated scale at a smaller batch size.

\subsection{Heavy Ball}
\label{sec:hb_analysis}
We can rewrite the HB update from equation~\eqref{eqn:hb_update} in the following form:
\begin{equation}
\label{eq:hb_update_both}
w_{t+1}=w_t-\eta(1-\beta)\,\hat g_t + \beta\,(w_t-w_{t-1}),
\qquad \eta>0,\ \beta\in[0,1). \notag
\end{equation}

We can rewrite the update as a linear recursion, where the vector form is the augmented state:
\begin{align}
\label{eqn:hb_recurrence_vector}
    \mqty[w_{t+1} - w^* \\w_t - w^*] &= \mqty[(1+\beta)I - \eta(1-\beta)\bar{X}_t & -\beta I\\I & 0]\mqty[w_{t} - w^* \\w_{t-1} - w^*] 
\end{align}

Denote by $z_t = \mqty[w_{t} - w^* \\w_{t-1} - w^*]$ with covariance $\Sigma_t = \mathop{\mathbb{E}}[z_t z_t ^\top]$, and denote the transition operator from equation~\eqref{eqn:hb_recurrence_vector} as $\hat{A}_t$. We can decompose $\hat{A}_t = M_t + \tilde{A}$, where $\tilde{A}$ will be the deterministic component and $M_t$ the random component as:
\begin{align*}
    M_t &= \mqty[-\eta(1-\beta)(\bar{X}_t - H) & 0 \\ 0 & 0]  && 
    \tilde{A} = \mqty[(1+\beta)I - \eta(1-\beta)H & -\beta I \\ I & 0]
\end{align*}

Computing the covariance of the augmented state vector gives:
\begin{align*}
    \Sigma_{t+1} = \tilde{A} \Sigma_t \tilde{A}^\top + \E [M_t z_t z_t^\top M_t^\top]
\end{align*}
For the second expectation, we need to compute a 4th moment Gaussian term coming from $\E[\bar{X}_t z_t z_t^\top \bar{X}_t]$, only on the $11$ block (since that is where $\bar{X}_t$ is). We can compute this term by applying Proposition~\ref{prop:gaussian_forth} and upper bound $H \Sigma H \preceq \Tr(H\Sigma)H$. With an abuse of notation, we will write the recursion of $\Sigma_{t+1}$ with equality after applying this bound, since we only lose a small constant factor, thus obtaining: 
\begin{align}
    \Sigma_{t+1} = \tilde{A} \Sigma_t \tilde{A}^\top + \mqty[\frac{2 \eta^2 (1-\beta)^2}{B} \Tr(H\Sigma_t^{11}) H & 0 \\0 & 0] \notag
\end{align}

Henceforth, we will express all matrices in the eigenbasis of \(H\). Thus, after rotation $H$ reduces to $\Lambda$ with $S_t = Q^\top\Sigma_tQ$ and $A = Q^\top \tilde{A} Q$. Thus, the recursion becomes:
\begin{equation}
\label{eqn:hb_recursion_operator}
S_{t+1}
= A\,S_t\,A^\top\;
+\;\eta^2(1-\beta)^2
\begin{bmatrix}
\tfrac{1}{B}\,\Lambda\,\mathrm{Tr}(\Lambda S^{11}_{t}) & 0\\[3pt]
0 & 0
\end{bmatrix}\
\end{equation}

Note that $A$ has a block diagonal structure as \(A=\mathrm{blkdiag}(A_1,\ldots,A_d)\), where each per–coordinate \(2\times2\) block is:
\begin{equation}
\label{eq:A-i}
A_i=
\begin{bmatrix}
a_i & -\beta\\[2pt]
1 & 0
\end{bmatrix},
\qquad
a_i:=(1+\beta)-\eta(1-\beta)\lambda_i .
\end{equation}

Let $\mathcal{T}_{H,\eta,\beta,B}$ denote the corresponding linear covariance
update operator on symmetric $2d\times 2d$ matrices as defined in ~\eqref{eqn:hb_recursion_operator}, such that:
\begin{equation*}
S_{t+1} = \mathcal{T}_{H,\eta,\beta,B}(S_t)
\end{equation*}

Let $s(H,\eta,\beta,B)$ denote the spectral radius of
$\mathcal{T}_{H,\eta,\beta,B}$ 
\[
s(H,\eta,\beta,B)
:= \rho\big(\mathcal{T}_{H,\eta,\beta,B}\big)
:= \max\{\,|s| : s\ \text{eigenvalue of}\ \mathcal{T}_{H,\eta,\beta,B}\,\}.
\]
As we see later on (in Lemma~\ref{lem:PF}), this spectral radius is attained by a real,
nonnegative eigenvalue with a PSD eigenmatrix. We define the \emph{optimal} spectral radius attainable at batch size $B$ as
\[
s^*(H,B) := \inf_{\eta>0,\ \beta\in[0,1)} s(H,\eta,\beta,B),
\]
This brings us to our first main result.

\begin{theorem}[HB Compute Efficiency Lower Bound]
\label{thm:hb_lower_bound}
For any data covariance matrix $H$ and mini-batch size $B\ge 1$, $\beta,\eta>0$, the optimal
spectral radius of HB satisfies:
\[
\qquad
s^*(H,B)
\ \gtrsim\ 1-\,\min\!\left\{
\frac{B\,\lambda_{\min}}{\mathrm{tr}(H)}\,,\
\sqrt{\frac{\lambda_{\min}}{\lambda_{\max}}}
\right\}.
\qquad
\]
where $\gtrsim$ absorbs universal constants.
In particular, the transition to the accelerated regime occurs at batch size $B_{\mathrm{HB}}^{\mathrm{crit}} = \frac{\Tr(H)\sqrt{\kappa}}{\lambda_{\max}}.$
\end{theorem}

Theorem~\ref{thm:hb_lower_bound} establishes a lower bound on the optimal spectral radius of HB, which in the case of consistent linear systems governs the error rate. From the theorem, we can see that there are, in effect, $2$ regimes for HB: for batch size $B < B_{\mathrm{HB}}^{\mathrm{crit}}$, the algorithm cannot improve over the SGD scaling, whereas for batch size $B \geq B_{\mathrm{HB}}^{\mathrm{crit}}$ the best-achievable spectral gap is bounded above by the asymptotic contraction rate. Therefore, one can conclude that in order to minimize the number of serial steps without wasting compute, the optimal batch size to train at is $B_{\mathrm{HB}}^{\mathrm{crit}}$. In Corollary~\ref{cor:hb_power_law}, we specialize the HB lower bound to a power law spectrum, in order to compare the rate with the ASGD bound in Section~\ref{sec:asgd_analysis}.
\begin{corollary}[HB on power-law spectra]
\label{cor:hb_power_law}
Assume the eigenvalues of $H$ satisfy $\lambda_i \eqsim i^{-a}$ for some $a>1$. Then, the optimal spectral radius of HB satisfies:
\[
s^*(H,B)
\;\gtrsim\;
1 - \min\!\left\{
B d^{-a},\ d^{-a/2}
\right\},
\]
where $\gtrsim$ absorbs universal constants.
In particular, the transition to the accelerated regime occurs at batch size $
B_{\mathrm{HB}}^{\mathrm{crit}}
\eqsim
d^{a/2}$
\end{corollary}
We defer the full proofs to Appendix~\ref{sec:appendix_hb_proofs}. In the following section, we establish compute efficiency lower bounds for ASGD.
\subsection{Accelerated SGD}
\label{sec:asgd_analysis}
The analysis for ASGD follows a very similar pattern as HB in Section~\ref{sec:hb_analysis}. 
Note that we assume $\zeta>0$ and $0 < \beta < 1$. The ASGD update rule from Equation~\ref{eqn:asgd_update} can be written as a linear recursion in the following augmented state:
\begin{align}
\label{eqn:rec_vec_asgd}
    \mqty[w_{t+1} - w^* \\w_{t+1} - w^* + \eta \mu_t] &= \mqty[(1+\beta)I - \eta(\zeta + 1-\beta)\bar{X}_t & -\beta I\\I - \eta \zeta \bar{X}_t & 0]\mqty[w_{t} - w^* \\w_{t} - w^* + \eta \mu_{t-1}]
\end{align}
We again decompose the transition matrix into a deterministic and a stochastic component:
\begin{align*}
    M_t
=
\begin{bmatrix}
-\eta(\zeta + 1 - \beta)(\bar X_t - H) & 0 \\
-\eta \zeta (\bar X_t - H) & 0
\end{bmatrix}
&&
\tilde{A}
=
\begin{bmatrix}
(1+\beta)I - \eta(\zeta + 1 - \beta)\,H & -\beta I \\
I - \eta \zeta H & 0
\end{bmatrix}
\end{align*}

Note that all the randomness is in $M_t$. Denoting by $z_t=
\begin{bmatrix}
w_t - w^* \\
w_t - w^* + \eta \mu_{t-1}
\end{bmatrix}$ and computing its covariance we get:
\begin{align*}
    \Sigma_{t+1} = A \Sigma_t A^T + \frac{\eta^2}{B} \mqty[(\zeta + (1-\beta))^2 \Tr(H \Sigma^{11}_t)H & \zeta(\zeta + 1 - \beta) \Tr(H \Sigma^{11}_t)H \\ \zeta(\zeta + 1 - \beta) \Tr(H \Sigma^{11}_t)H & \zeta^2 \Tr(H \Sigma^{11}_t)H]
\end{align*}

Similarly, we will express all matrices in the eigenbasis of \(H\). Reusing the same notation as in Section~\ref{sec:hb_analysis}  after rotation in the eigenbasis of $H$ we get $\Lambda$ and $S_t = Q^\top\Sigma_tQ$ and $A = Q^\top \tilde{A} Q$. After computing the 4th moment term using Proposition~\ref{prop:gaussian_forth}, the recursion becomes:
\begin{equation}
S_{t+1}
=
A S_t A^\top
+
\frac{\eta^2}{B}
\begin{bmatrix}
(\zeta + (1-\beta))^2 \,\Tr(\Lambda S^{11}_{t})\, \Lambda
&
\zeta(\zeta + 1-\beta)\,\Tr(\Lambda S^{11}_{t})\, \Lambda
\\[4pt]
\zeta(\zeta + 1-\beta)\,\Tr(\Lambda S^{11}_{t})\, \Lambda
&
\zeta^2 \,\Tr(\Lambda S^{11}_{t})\, \Lambda
\end{bmatrix}
\label{eqn:st_recur_zeta}
\end{equation}

Let $\mathcal{T}_{H,\eta,\beta,\zeta,B}$ denote the corresponding linear covariance update operator for this update rule. Then, we have the following statement.
\begin{theorem}[ASGD Lower Bound under power-law spectra]
\label{thm:asgd-powerlaw}
Assume the eigenvalues of $H$ satisfy $\lambda_i \eqsim i^{-a}$ for some $a>1$.
Then the optimal spectral radius of ASGD satisfies:
\[
s^*(H,B)
\;\gtrsim\;
1 -
\begin{cases}
B\, d^{-\frac{a^2}{2a-1}}, & B \lesssim 1,\\[6pt]
B^{\frac{a}{2a-1}}\, d^{-\frac{a^2}{2a-1}}, & 1 \lesssim B \lesssim d^{1/2},\\[6pt]
d^{-a/2}, & B \gtrsim d^{1/2},
\end{cases}
\]
where $\gtrsim$ absorbs universal constants. In particular, the lower bound saturates at the accelerated scale at batch size $
B_{\mathrm{ASGD}}^{\mathrm{crit}} \eqsim d^{1/2}$.
\end{theorem}
Theorem~\ref{thm:asgd-powerlaw} distinguishes $3$ scaling regimes as a function of batch size. In the first \textit{linear regime}, linearly increasing the batch size allows for a linear decrease (up to constants) in the serial steps necessary to achieve a fixed target error. After this, there is a \textit{diminishing returns} regime indicating that linearly increasing the batch size gives us a sublinear decrease in the serial runtime, since for $a>1$ the batch exponent $a/(2a-1) < 1$ (note that for large $a$ this term behaves roughly like $\sqrt{B}$). Finally, for large batch size we reach a \textit{saturation} regime, saturating the lower bound at the accelerated scale. Increasing the batch size any further would not yield any improvements in the serial runtime, and will be, instead, wasting compute. 

However, note that in the power law setting, the ASGD lower bound reaches the accelerated scale at a \textit{smaller} batch size than the HB lower bound. We expand upon this finding in Corollary~\ref{cor:asgd-earlier-than-hb}.

\begin{corollary}[ASGD Accelerated Regime.]
\label{cor:asgd-earlier-than-hb}
Assume the eigenvalues of $H$ satisfy $\lambda_i \eqsim i^{-a}$ for some $a>1$.
Let $B_{\mathrm{HB}}^{\mathrm{crit}}$ and $B_{\mathrm{ASGD}}^{\mathrm{crit}}$ denote the smallest batch sizes at which the HB lower bound and the ASGD lower bound, respectively, reach the optimal accelerated spectral gap $\Theta(d^{-a/2})$. Then
\[
B_{\mathrm{HB}}^{\mathrm{crit}} \eqsim d^{a/2},
\qquad
B_{\mathrm{ASGD}}^{\mathrm{crit}} \eqsim d^{1/2}.
\]
Hence
\[
\frac{B_{\mathrm{HB}}^{\mathrm{crit}}}{B_{\mathrm{ASGD}}^{\mathrm{crit}}}
\eqsim d^{(a-1)/2},
\]
so for every $a>1$, the ASGD lower bound reaches the accelerated scale at a strictly smaller batch size than the HB lower bound.
Equivalently, throughout the interval
\[
d^{1/2} \;\lesssim\; B \;\lesssim\; d^{a/2},
\]
the ASGD lower bound has already saturated at the accelerated scale.
\end{corollary}
We defer the full proofs to Appendix~\ref{sec:appendix_asgd_proofs}.

\paragraph{Compute-serial runtime tradeoff.} We now summarize the implications of the lower bounds above. Theorem~\ref{thm:hb_lower_bound} shows that, for arbitrary spectra, HB cannot improve over the SGD CE frontier. However, note that HB can use larger batches to reduce
serial runtime, but this does not correspond to an improved compute-efficiency
frontier over SGD. Beyond the HB critical batch size, the lower bound
saturates at the deterministic accelerated scale. 

Under power-law spectra with \(a>1\), Theorem~\ref{thm:asgd-powerlaw}
shows that ASGD has a better small-batch spectral-gap scaling than HB and that
its lower bound reaches the accelerated scale at the smaller batch size
\(B_{\mathrm{ASGD}}^{\mathrm{crit}}\asymp d^{1/2}\), as stated in
Corollary~\ref{cor:asgd-earlier-than-hb}. Thus ASGD offers a sharper
compute--serial tradeoff: small batches preserve its compute-efficiency
advantage, while increasing the batch size up to
\(B_{\mathrm{ASGD}}^{\mathrm{crit}}\) converts part of this advantage into
reduced serial runtime. We empirically confirm these regimes in synthetic
linear-regression experiments in Section~\ref{sec:experiments}.
\section{Experiments}
\label{sec:experiments}

We run synthetic experiments in linear regression on quadratics with power law data with $a = 2.0$, showing that the observed scaling is consistent with the lower bounds. We set the problem size to be $D=50000$ and train for $N = 500000$, at batch sizes $B \in \{1, 2, 4, 8, 16, 32, 64, 128, 256, 1024\}$, averaged over $50$ seeds. For both HB and ASGD we sweep over learning rate $\eta \in \{10^{-5}, 3\cdot 10^{-5}, 10^{-4}, 3\cdot 10^{-4}, 10^{-3}, 3\cdot 10^{-3}, 10^{-2}, 3\cdot 10^{-2}, 10^{-1}, 3\cdot 10^{-1}, 1.0, 2.0, 3.0, 5.0, 10.0\}$, momentum EMA parameter $\beta \in \{0.8, 0.9, 0.95, 0.99, 0.999, 0.9999\}$ and ASGD hyperparameter $\zeta \in \{0.05, 0.1, 0.2, 0.3, 0.5, 0.7, 0.9, 0.95, 0.99\}$, with $\zeta = 0$ for HB, and we plot each curve at the best set of hyperparameters. We plot the number of steps required to reach a target loss as a function of the batch size. 
\section{Discussion and Conclusions}
In this work, we have established compute efficiency lower bounds for heavy ball momentum and accelerated SGD, as a function of the problem instance and the batch size. Specializing the problem instances to power law spectra, we have directly compared the $2$ algorithms showing that there is a performance to serial runtime tradeoff: one can train with ASGD at a smaller batch size for a longer serial runtime, achieving better final loss. From a theoretical perspective, our lower bound for ASGD improves over that of~\citep{kidambi2018insufficiency}, most notably due to the bound applying for any power law spectrum and not a specifically constructed problem. We believe that the techniques used in this work to derive the lower bounds can also be applied to deriving upper bounds, and we leave this derivation to future work. 

\section*{Acknowledgements}
The authors would like to thank Alex Damian and Jingfeng Wu, for helpful discussions. The authors would also like to thank Max Shad and Bala Desinghu for their help with the cluster. AM, DM, PN acknowledge the support of a Kempner Institute Graduate Research Fellowship. AM, SK, DM and PN
acknowledge that this work has been made possible in part by a gift from the Chan Zuckerberg
Initiative Foundation to establish the Kempner Institute for the Study of Natural and Artificial
Intelligence. SK and DM acknowledge support from the Office of Naval Research under award N0001422-1-2377 and the National Science Foundation Grant under award \#IIS 2229881. DM is also supported by a Simons Investigator Fellowship, NSF grant DMS-2134157, DARPA grant W911NF2010021,and DOE grant DE-SC0022199.

\bibliographystyle{plainnat}
\bibliography{references}

@inproceedings{kidambi2018insufficiency,
  title={On the insufficiency of existing momentum schemes for stochastic optimization},
  author={Kidambi, Rahul and Netrapalli, Praneeth and Jain, Prateek and Kakade, Sham},
  booktitle={2018 Information Theory and Applications Workshop (ITA)},
  pages={1--9},
  year={2018},
  organization={IEEE}
}

@article{robbins1951stochastic,
  title={A stochastic approximation method},
  author={Robbins, Herbert and Monro, Sutton},
  journal={The annals of mathematical statistics},
  pages={400--407},
  year={1951},
  publisher={JSTOR}
}

@article{grattafiori2024llama,
  title={The llama 3 herd of models},
  author={Grattafiori, Aaron and Dubey, Abhimanyu and Jauhri, Abhinav and Pandey, Abhinav and Kadian, Abhishek and Al-Dahle, Ahmad and Letman, Aiesha and Mathur, Akhil and Schelten, Alan and Vaughan, Alex and others},
  journal={arXiv preprint arXiv:2407.21783},
  year={2024}
}

@article{team2026kimi,
  title={Kimi K2. 5: Visual Agentic Intelligence},
  author={Team, Kimi and Bai, Tongtong and Bai, Yifan and Bao, Yiping and Cai, SH and Cao, Yuan and Charles, Y and Che, HS and Chen, Cheng and Chen, Guanduo and others},
  journal={arXiv preprint arXiv:2602.02276},
  year={2026}
}

@article{guo2025deepseek,
  title={Deepseek-r1: Incentivizing reasoning capability in llms via reinforcement learning},
  author={Guo, Daya and Yang, Dejian and Zhang, Haowei and Song, Junxiao and Wang, Peiyi and Zhu, Qihao and Xu, Runxin and Zhang, Ruoyu and Ma, Shirong and Bi, Xiao and others},
  journal={arXiv preprint arXiv:2501.12948},
  year={2025}
}

@article{yang2025qwen3,
  title={Qwen3 technical report},
  author={Yang, An and Li, Anfeng and Yang, Baosong and Zhang, Beichen and Hui, Binyuan and Zheng, Bo and Yu, Bowen and Gao, Chang and Huang, Chengen and Lv, Chenxu and others},
  journal={arXiv preprint arXiv:2505.09388},
  year={2025}
}

@article{kingma2014adam,
  title={Adam: A method for stochastic optimization},
  author={Kingma, Diederik P and Ba, Jimmy},
  journal={arXiv preprint arXiv:1412.6980},
  year={2014}
}

@article{loshchilov2017decoupled,
  title={Decoupled weight decay regularization},
  author={Loshchilov, Ilya and Hutter, Frank},
  journal={arXiv preprint arXiv:1711.05101},
  year={2017}
}

@article{vyas2024soap,
  title={Soap: Improving and stabilizing shampoo using adam},
  author={Vyas, Nikhil and Morwani, Depen and Zhao, Rosie and Kwun, Mujin and Shapira, Itai and Brandfonbrener, David and Janson, Lucas and Kakade, Sham},
  journal={arXiv preprint arXiv:2409.11321},
  year={2024}
}

@inproceedings{martens2015optimizing,
  title={Optimizing neural networks with kronecker-factored approximate curvature},
  author={Martens, James and Grosse, Roger},
  booktitle={International conference on machine learning},
  pages={2408--2417},
  year={2015},
  organization={PMLR}
}

@inproceedings{gupta2018shampoo,
  title={Shampoo: Preconditioned stochastic tensor optimization},
  author={Gupta, Vineet and Koren, Tomer and Singer, Yoram},
  booktitle={International Conference on Machine Learning},
  pages={1842--1850},
  year={2018},
  organization={PMLR}
}

@misc{jordan2024muon,
  author       = {Keller Jordan and Yuchen Jin and Vlado Boza and Jiacheng You and
                  Franz Cesista and Laker Newhouse and Jeremy Bernstein},
  title        = {Muon: An optimizer for hidden layers in neural networks},
  year         = {2024},
  url          = {https://kellerjordan.github.io/posts/muon/}
}

@article{polyak1964some,
  title={Some methods of speeding up the convergence of iteration methods},
  author={Polyak, Boris T},
  journal={Ussr computational mathematics and mathematical physics},
  volume={4},
  number={5},
  pages={1--17},
  year={1964},
  publisher={Elsevier}
}

@inproceedings{nesterov1983method,
  title={A method for unconstrained convex minimization problem with the rate of convergence O (1/k2)},
  author={Nesterov, Yurii},
  booktitle={Dokl. Akad. Nauk. SSSR},
  volume={269},
  pages={543},
  year={1983}
}

@book{nesterov2013introductory,
  title={Introductory lectures on convex optimization: A basic course},
  author={Nesterov, Yurii},
  volume={87},
  year={2013},
  publisher={Springer Science \& Business Media}
}

@inproceedings{jain2018accelerating,
  title={Accelerating stochastic gradient descent for least squares regression},
  author={Jain, Prateek and Kakade, Sham M and Kidambi, Rahul and Netrapalli, Praneeth and Sidford, Aaron},
  booktitle={Conference On Learning Theory},
  pages={545--604},
  year={2018},
  organization={PMLR}
}

@inproceedings{vaswani2019fast,
  title={Fast and faster convergence of sgd for over-parameterized models and an accelerated perceptron},
  author={Vaswani, Sharan and Bach, Francis and Schmidt, Mark},
  booktitle={The 22nd international conference on artificial intelligence and statistics},
  pages={1195--1204},
  year={2019},
  organization={PMLR}
}

@article{liu2018accelerating,
  title={Accelerating SGD with momentum for over-parameterized learning},
  author={Liu, Chaoyue and Belkin, Mikhail},
  journal={arXiv preprint arXiv:1810.13395},
  year={2018}
}

@article{gupta2024nesterov,
  title={Nesterov acceleration despite very noisy gradients},
  author={Gupta, Kanan and Siegel, Jonathan W and Wojtowytsch, Stephan},
  journal={Advances in Neural Information Processing Systems},
  volume={37},
  pages={20694--20744},
  year={2024}
}

@article{zhang2024does,
  title={How Does Critical Batch Size Scale in Pre-training?},
  author={Zhang, Hanlin and Morwani, Depen and Vyas, Nikhil and Wu, Jingfeng and Zou, Difan and Ghai, Udaya and Foster, Dean and Kakade, Sham},
  journal={arXiv preprint arXiv:2410.21676},
  year={2024}
}

@article{merrill2025critical,
  title={Critical batch size revisited: A simple empirical approach to large-batch language model training},
  author={Merrill, William and Arora, Shane and Groeneveld, Dirk and Hajishirzi, Hannaneh},
  journal={arXiv preprint arXiv:2505.23971},
  year={2025}
}

@article{meterez2025seesaw,
  title={Seesaw: Accelerating Training by Balancing Learning Rate and Batch Size Scheduling},
  author={Meterez, Alexandru and Morwani, Depen and Wu, Jingfeng and Oncescu, Costin-Andrei and Pehlevan, Cengiz and Kakade, Sham},
  journal={arXiv preprint arXiv:2510.14717},
  year={2025}
}

@article{pagliardini2024ademamix,
  title={The ademamix optimizer: Better, faster, older},
  author={Pagliardini, Matteo and Ablin, Pierre and Grangier, David},
  journal={arXiv preprint arXiv:2409.03137},
  year={2024}
}

@article{morwani2025connections,
  title={Connections between schedule-free optimizers, ademamix, and accelerated sgd variants},
  author={Morwani, Depen and Vyas, Nikhil and Zhang, Hanlin and Kakade, Sham},
  journal={arXiv preprint arXiv:2502.02431},
  year={2025}
}

@article{ferbach2025dimension,
  title={Dimension-adapted momentum outscales SGD},
  author={Ferbach, Damien and Everett, Katie and Gidel, Gauthier and Paquette, Elliot and Paquette, Courtney},
  journal={arXiv preprint arXiv:2505.16098},
  year={2025}
}

@article{lee2022trajectory,
  title={Trajectory of mini-batch momentum: batch size saturation and convergence in high dimensions},
  author={Lee, Kiwon and Cheng, Andrew and Paquette, Elliot and Paquette, Courtney},
  journal={Advances in Neural Information Processing Systems},
  volume={35},
  pages={36944--36957},
  year={2022}
}

@article{defazio2024road,
  title={The road less scheduled},
  author={Defazio, Aaron and Yang, Xingyu and Mehta, Harsh and Mishchenko, Konstantin and Khaled, Ahmed and Cutkosky, Ashok},
  journal={Advances in Neural Information Processing Systems},
  volume={37},
  pages={9974--10007},
  year={2024}
}

@article{chen2023symbolic,
  title={Symbolic discovery of optimization algorithms},
  author={Chen, Xiangning and Liang, Chen and Huang, Da and Real, Esteban and Wang, Kaiyuan and Pham, Hieu and Dong, Xuanyi and Luong, Thang and Hsieh, Cho-Jui and Lu, Yifeng and others},
  journal={Advances in neural information processing systems},
  volume={36},
  pages={49205--49233},
  year={2023}
}

@article{lucas2018aggregated,
  title={Aggregated momentum: Stability through passive damping},
  author={Lucas, James and Sun, Shengyang and Zemel, Richard and Grosse, Roger},
  journal={arXiv preprint arXiv:1804.00325},
  year={2018}
}

@article{ma2018quasi,
  title={Quasi-hyperbolic momentum and Adam for deep learning},
  author={Ma, Jerry and Yarats, Denis},
  journal={arXiv preprint arXiv:1810.06801},
  year={2018}
}

@inproceedings{wu2022last,
  title={Last iterate risk bounds of sgd with decaying stepsize for overparameterized linear regression},
  author={Wu, Jingfeng and Zou, Difan and Braverman, Vladimir and Gu, Quanquan and Kakade, Sham},
  booktitle={International conference on machine learning},
  pages={24280--24314},
  year={2022},
  organization={PMLR}
}

@inproceedings{zou2021benign,
  title={Benign overfitting of constant-stepsize sgd for linear regression},
  author={Zou, Difan and Wu, Jingfeng and Braverman, Vladimir and Gu, Quanquan and Kakade, Sham},
  booktitle={Conference on learning theory},
  pages={4633--4635},
  year={2021},
  organization={PMLR}
}

@article{wu2022power,
  title={The power and limitation of pretraining-finetuning for linear regression under covariate shift},
  author={Wu, Jingfeng and Zou, Difan and Braverman, Vladimir and Gu, Quanquan and Kakade, Sham},
  journal={Advances in Neural Information Processing Systems},
  volume={35},
  pages={33041--33053},
  year={2022}
}

@article{meterez2025simplified,
  title={A simplified analysis of sgd for linear regression with weight averaging},
  author={Meterez, Alexandru and Morwani, Depen and Oncescu, Costin-Andrei and Wu, Jingfeng and Pehlevan, Cengiz and Kakade, Sham},
  journal={arXiv preprint arXiv:2506.15535},
  year={2025}
}

@article{cauchy1847methode,
  title={M{\'e}thode g{\'e}n{\'e}rale pour la r{\'e}solution des systemes d’{\'e}quations simultan{\'e}es},
  author={Cauchy, Augustin and others},
  journal={Comp. Rend. Sci. Paris},
  volume={25},
  number={1847},
  pages={536--538},
  year={1847}
}

@article{shallue2019measuring,
  title={Measuring the effects of data parallelism on neural network training},
  author={Shallue, Christopher J and Lee, Jaehoon and Antognini, Joseph and Sohl-Dickstein, Jascha and Frostig, Roy and Dahl, George E},
  journal={Journal of Machine Learning Research},
  volume={20},
  number={112},
  pages={1--49},
  year={2019}
}

@article{fu2023and,
  title={When and why momentum accelerates sgd: An empirical study},
  author={Fu, Jingwen and Wang, Bohan and Zhang, Huishuai and Zhang, Zhizheng and Chen, Wei and Zheng, Nanning},
  journal={arXiv preprint arXiv:2306.09000},
  year={2023}
}

@article{kunstner2023noise,
  title={Noise is not the main factor behind the gap between sgd and adam on transformers, but sign descent might be},
  author={Kunstner, Frederik and Chen, Jacques and Lavington, Jonathan Wilder and Schmidt, Mark},
  journal={arXiv preprint arXiv:2304.13960},
  year={2023}
}

@article{wang2023marginal,
  title={The marginal value of momentum for small learning rate sgd},
  author={Wang, Runzhe and Malladi, Sadhika and Wang, Tianhao and Lyu, Kaifeng and Li, Zhiyuan},
  journal={arXiv preprint arXiv:2307.15196},
  year={2023}
}

@article{marek2025small,
  title={Small batch size training for language models: When vanilla sgd works, and why gradient accumulation is wasteful},
  author={Marek, Martin and Lotfi, Sanae and Somasundaram, Aditya and Wilson, Andrew Gordon and Goldblum, Micah},
  journal={arXiv preprint arXiv:2507.07101},
  year={2025}
}

@article{zhang2019algorithmic,
  title={Which algorithmic choices matter at which batch sizes? insights from a noisy quadratic model},
  author={Zhang, Guodong and Li, Lala and Nado, Zachary and Martens, James and Sachdeva, Sushant and Dahl, George and Shallue, Chris and Grosse, Roger B},
  journal={Advances in neural information processing systems},
  volume={32},
  year={2019}
}

@article{bordelon2026theory,
  title={Theory of Optimal Learning Rate Schedules and Scaling Laws for a Random Feature Model},
  author={Bordelon, Blake and Mori, Francesco},
  journal={arXiv preprint arXiv:2602.04774},
  year={2026}
}

@article{cohen2022adaptive,
  title={Adaptive gradient methods at the edge of stability},
  author={Cohen, Jeremy M and Ghorbani, Behrooz and Krishnan, Shankar and Agarwal, Naman and Medapati, Sourabh and Badura, Michal and Suo, Daniel and Cardoze, David and Nado, Zachary and Dahl, George E and others},
  journal={arXiv preprint arXiv:2207.14484},
  year={2022}
}

@article{cohen2021gradient,
  title={Gradient descent on neural networks typically occurs at the edge of stability},
  author={Cohen, Jeremy M and Kaur, Simran and Li, Yuanzhi and Kolter, J Zico and Talwalkar, Ameet},
  journal={arXiv preprint arXiv:2103.00065},
  year={2021}
}

@article{andreyev2026momentum,
  title={Momentum Further Constrains Sharpness at the Edge of Stochastic Stability},
  author={Andreyev, Arseniy and Ananthkumar, Advikar and Walden, Marc and Poggio, Tomaso and Beneventano, Pierfrancesco},
  journal={arXiv preprint arXiv:2604.14108},
  year={2026}
}
\clearpage
\appendix

\section{Heavy Ball Analysis}
\label{sec:appendix_hb_proofs}

\subsection{Proof of Proposition \ref{prop:gaussian_forth}}
\begin{proof}
    The proof is a simple application of Isserlis's theorem. Elementwise, we have that:
    \[
    \E [xx^\top \Sigma xx^\top]_{ij} = \sum_{kl} \Sigma_{kl} \E[x_i x_k x_l x_j] = \sum_{kl} \Sigma_{kl} (H_{ik} H_{lj} + H_{il} H_{kj} + H_{ij} H_{kl})
    \]
    Assembling the result in matrix form gives the first part.

    For \(v \in \mathbb{R}^d\), define \(u := H^{1/2}v\). Then
\[
v^\top H \Sigma H v = v^\top H^{1/2}(H^{1/2} \Sigma H^{1/2})H^{1/2} v = u^\top H^{1/2} \Sigma H^{1/2} u.
\]
Also,
\[
\Tr(H \Sigma)\, v^\top H v = \Tr(H^{1/2}\Sigma H^{1/2})\, u^\top u.
\]
Since \(H \succeq 0\) and \(\Sigma \succeq 0\), we have \(H^{1/2}\Sigma H^{1/2} \succeq 0\). Hence there exists an orthogonal matrix \(P\) and a diagonal matrix \(D \succeq 0\) such that
\[
H^{1/2}\Sigma H^{1/2} = P D P^\top.
\]
Since \(D \succeq 0\), we have \(D \preceq \Tr(D)\, I\). Therefore,
\[
u^\top H^{1/2}\Sigma H^{1/2} u = u^\top P D P^\top u \le u^\top P(\Tr(D) I)P^\top u = \Tr(D)\, u^\top u.
\]
Using \(\Tr(D)=\Tr(H^{1/2}\Sigma H^{1/2})=\Tr(H\Sigma)\), we obtain
\[
v^\top H \Sigma H v \le \Tr(H\Sigma)\, v^\top H v.
\]
Since this holds for all \(v \in \mathbb{R}^d\), it follows that $H \Sigma H \preceq \Tr(H \Sigma) H$.

\end{proof}

\subsection{Deriving the Secular Equation}
Recall that:
\begin{equation*}
S_{t+1}
= A\,S_t\,A^\top\;
+\;\eta^2(1-\beta)^2
\begin{bmatrix}
\tfrac{1}{B}\,\Lambda\,\mathrm{Tr}(\Lambda S^{11}_{t}) & 0\\[3pt]
0 & 0
\end{bmatrix}\
\end{equation*}
Let $\gamma = \mathrm{Tr}(\Lambda S_t^{11})$. Pushing a $\vec$ through this equation coupled with the fact that $\vec(AS_tA^\top) = (A \otimes A)\vec(S_t)$ gives us:

\begin{align}
    s \vec(S_t) &= (A \otimes A) \vec(S_t) + \frac{2 \eta^2 (1-\beta)^2 \gamma }{B} \mqty[\vec(\Lambda) \\0 \\0\\0]  \notag
\end{align}
\begin{align}
        \implies \qty(sI - A \otimes A) \vec(S_t) &= \frac{2 \eta^2 (1-\beta)^2 \gamma }{B} \mqty[\vec(\Lambda) \\0 \\0\\0] \notag
\end{align}

We can break the above equation up into $d$, $2 \times 2$ equations each using a block of $A$:

\begin{align*}
    (sI - A_i \otimes A_i) \vec(S_{t,i}) = \frac{2 \eta^2 (1-\beta)^2 \gamma }{B} \mqty[\lambda_i \\0 \\0\\0]
\end{align*}

\begin{align}
    \implies \vec(S_{t,i}) &= \frac{2 \eta^2 (1-\beta)^2 \gamma }{B} (sI - A_i \otimes A_i)^{-1} \mqty[\lambda_i \\0 \\0\\0] \notag
\end{align}

To simplify $\gamma$, we multiply the above Equation by $\mqty[\lambda_i \\0 \\0\\0]$ and sum over $i \in \{1, \dots, d\}$ to get:

\begin{align}
\label{eqn:recur_sum}
    1 = \frac{2 \eta^2 (1-\beta)^2}{B}\sum_{i=1}^d \mqty[\lambda_i & 0 & 0 & 0](sI - A_i \otimes A_i)^{-1} \mqty[\lambda_i \\0 \\0\\0]
\end{align}

\paragraph{Inverting the Kronecker.} Suppose $M_i$ is an eigenbasis for $A_i$, and $D_i$ are the eigenvalues of $A_i$. Then, $M_i \otimes M_i$ is an eigenbasis for $A_i \otimes A_i$ and $D_i \otimes D_i$ are the eigenvalues of $A_i$. Thus, we have:

\begin{align}
    (s I - A_i \otimes A_i)^{-1} = M_i \otimes M_i (sI - D_i \otimes D_i)^{-1} M_i^\top \otimes M_i^\top \notag
\end{align}
with
\begin{align*}
    D_i = \mqty[r_{i,+} & 0 \\ 0 &r_{i,-}] \notag
\end{align*}
and
\begin{align}
    M_i = \mqty[r_{i,+} & r_{i,-} \\ 1 & 1] \label{eqn:kron_inv}
\end{align}

Equations~\labelcref{eqn:kron_inv,eqn:recur_sum} on further simlification give us:

\begin{align*}
    1 = \frac{2 \eta^2 (1-\beta)^2}{B}\sum_i \frac{\lambda_i^2 s (s+\beta)}{(s - r_{i,+}^2)(s - r_{i,-}^2)(s-\beta)}
\end{align*}
Let
\begin{equation*}
\phi_i(s)
=\frac{s(s+\beta)}{(s-\beta)\,(s-r_{i,+}^2)\,(s-r_{i,-}^2)} .
\end{equation*}
The final secular Equation can be writen as
\begin{equation}
\label{eq:secular-exact}
\qquad
1\;=\;\frac{c}{B}\sum_{i=1}^d
\frac{\lambda_i^2\,\phi_i(s)}{1- c\lambda_i^2 \phi_i(s)}\,,\qquad
c=\eta^2(1-\beta)^2.
\qquad
\end{equation}

\textbf{Validity of vectorizing Equation~\ref{eqn:st_recur_zeta}} The vectorized equation is acting on the $4d$ dimensional space defined by the $d$ $2\times2$ blocks, while the symmetric matrices occupy a $3d$ space here. So the secular equation has extra roots. But we will argue that the maximum eigenvalue of the secular equation is still associated with a per-diagonal block PSD matrix.

First, within the space of per-diagonal symmetric matrices, per-diagonal PSD matrix has the maximum eigenvalue. This follows from KT theorem, as per-diagonal PSD matrix form a total cone. Then, let's consider the solutions of secular equation outside the symmetric space. Let $M$ be an eigen vector of $\mathcal{T}_{H,\eta,\beta,B}$ which is neither symmetric nor anti-symmetric. By linearity of $\mathcal{T}_{H,\eta,\beta,B}$,

\[  \mathcal{T}_{H,\eta,\beta,B}(M) = \lambda M  \implies  \mathcal{T}_{H,\eta,\beta,B}(M^T) =\lambda M^\top\]

Now, let's decompose \(M\) into its symmetric and anti-symmetric parts:
\[
M_s:=\frac{M+M^\top}{2},
\qquad
M_a:=\frac{M-M^\top}{2}.
\]
Then, by linearity,
\[
\mathcal{T}_{H,\eta,\beta,B}(M_s)
=
\frac{\mathcal{T}_{H,\eta,\beta,B}(M)+\mathcal{T}_{H,\eta,\beta,B}(M^\top)}{2}
=
\frac{\lambda M+\lambda M^\top}{2}
=
\lambda X_s,
\]
and similarly
\[
\mathcal{T}_{H,\eta,\beta,B}(M_a)
=
\frac{\mathcal{T}_{H,\eta,\beta,B}(M)-\mathcal{T}_{H,\eta,\beta,B}(M^\top)}{2}
=
\frac{\lambda M-\lambda M^\top}{2}
=
\lambda M_a.
\]
Therefore, any eigenvalue carried by an eigenmatrix $M$ is also carried by a symmetric eigenmatrix and an anti-symmetric eigenmatrix. Also, if $M$ is anti-symmetric, $\mathrm{Tr(\Lambda S^{11}) = 0} \implies \mathcal{T}_{H,\eta,\beta,B}(M) = \mathcal{T}_{\infty}(M).$ Thus, the eigenvalue of an anti-symmetric matrix for the stochastic operator coincide with the eigenvalue for the deterministic operator. In addition, note the following: (i) $A$ is block diagonal, (ii) give the stochastic opertor $\mathcal{T}_{H,\eta,\beta,B}$, gauging at the block diagonal eigenmatrices of the determinsitc operator $\mathcal{T}_{\infty}$ suffices. For a  block diagonal eigenmatrix, the eigenvalues of  $\mathcal{T}_{\infty}$ are that of $A_i \kron A_i$, with the eigenmatrix corresponding to the largest eigenvalue being a symmetric eigenmatrix, with rank-1 block diagonal entries of the form $u_i u_i^T$, where $u_i$ represents the eigenvector of $A_i$ corresponding to its largest eigenvalue. Thus, on the determinisitc operator $\mathcal{T}_{\infty}$, the eigenvalue of a symmetric eigenmatrix is always greater than or equal to that of an anti-symmetric matrix. By Conjecture~\ref{lem:PF}, the spectral radius for $\mathcal{T}_{H,\eta,\beta,B}$ is obtained by a real eigenvalue and from Lemma~\ref{lem:det_upper_bound}, the spectral radius of $\mathcal{T}_{H,\eta,\beta,B}$ is lower bounded by the spectral radius for $\mathcal{T}_{\infty}$, thus, the dominating eigenvalue of the secular equation is given by a symmetric PSD matrix.

\subsection{The Zero-Noise Heavy Ball Analysis}
In the deterministic setting, the Heavy-Ball iteration in Equation~\ref{eqn:hb_recursion_operator} reduces to 
\begin{align}
\label{eqn:recur_sigma_t_det}
    S_{t+1} = A S_t A^\top
\end{align}

The following technical lemma governs the
maximal learning rate of the deterministic Heavy--Ball algorithm. 

\begin{lemma}[Stability band of deterministic HB]
\label{lem:HB-det-band}
Assume $\beta,\eta>0$. If the Heavy--Ball iteration in Equation~\ref{eqn:recur_sigma_t_det} is stable
(i.e.\ all eigenvalues of every $A_i$ lie in the open unit disk), if and only if
\[
\beta < 1,
\qquad
0 < \eta < \frac{2(1+\beta)}{(1-\beta)\lambda_{\max}}.
\]
\end{lemma}

\begin{proof}
Fix $\beta\in[0,1)$ and first consider a generic $2\times2$ matrix
\[
A =
\begin{bmatrix}
a & -\beta\\[2pt]
1 & 0
\end{bmatrix},
\qquad a\in\mathbb{R}.
\]
Let $r_{\pm}$ be the eigenvalues of $A$, i.e.\ the roots of
\[
z^2 - a z + \beta \;=\; 0.
\]
We recall the standard discrete-time stability criterion for a degree–$2$
polynomial. Writing this polynomial in the form
\[
z^2 + p z + q = 0
\]
with $p=-a$ and $q=\beta$, the Jury/Schur stability test states that the
roots lie strictly inside the unit disk, $|z|<1$, if and only if
\[
|q| < 1,\qquad
1 + p + q > 0,\qquad
1 - p + q > 0,\qquad
1 - q > 0.
\]
(These are the degree–$2$ Jury conditions.)

In our case, substituting $p=-a$ and $q=\beta$ into the above
gives
\begin{align*}
|q| < 1
&\iff |\beta| < 1,\\[2pt]
1 + p + q > 0
&\iff 1 - a + \beta > 0 \iff a < 1+\beta,\\[2pt]
1 - p + q > 0
&\iff 1 + a + \beta > 0 \iff a > -(1+\beta),\\[2pt]
1 - q > 0
&\iff 1 - \beta > 0.
\end{align*}
Thus, for the matrix $A$, the
eigenvalues satisfy $|r_{\pm}|<1$ if and only if $|\beta|<1$ and
\[
-(1+\beta)\;<\;a\;<\;1+\beta.
\]

We now apply this to the deterministic HB iteration on the
quadratic. In particular, we apply these conditions to each state
matrix $A_i$, where $
a_i = (1+\beta) - \eta(1-\beta)\lambda_i$.
Stability for all $i$ requires 
\[
-(1+\beta) < a_i < 1+\beta
\quad\text{for all}\quad \lambda_i\in[0,\lambda_{\max}].
\]

The upper bound $a_i<1+\beta$ is satisfied, since $\eta>0$ by assumption.
The lower bound must hold in particular at the largest curvature $\lambda_{\max}$, where $a_i$ is
smallest:
\[
(1+\beta) - \eta(1-\beta)\lambda_{\max} \;>\; -(1+\beta).
\]
Rearranging,
\[
(1+\beta) - \eta(1-\beta)\lambda_{\max} > -(1+\beta)
\quad\Longleftrightarrow\quad
2(1+\beta) > \eta(1-\beta)\lambda_{\max}
\quad\Longleftrightarrow\quad
\eta < \frac{2(1+\beta)}{(1-\beta)\lambda_{\max}}.
\]

Finally, stability also requires $|\beta|<1$ and $1-\beta>0$, which in
our nonnegative-$\beta$ setting is exactly $0\le \beta<1$.
Combining these conditions gives the claimed deterministic HB stability band.
\end{proof}

\subsection{Compute Efficiency (CE) lower bounds}

Let $s(H,\eta,\beta,B)$ denote the spectral radius of
$\mathcal{T}_{H,\eta,\beta,B}$:
\[
s(H,\eta,\beta,B)
:= \rho\big(\mathcal{T}_{H,\eta,\beta,B}\big)
:= \max\{\,|s| : s\ \text{eigenvalue of}\ \mathcal{T}_{H,\eta,\beta,B}\,\}.
\]
As we see later on (in Lemma~\ref{lem:PF}), this spectral radius is attained by a real,
nonnegative eigenvalue with a PSD eigenmatrix.

We define the \emph{optimal} spectral radius attainable at batch size $B$ as
\[
s^*(H,B) := \inf_{\eta>0,\ \beta\in[0,1)} s(H,\eta,\beta,B),
\]
and the corresponding \emph{spectral gap} as $1 - s^*(H,B)$.

\begin{theorem}[HB--SGD compute efficiency lower bound]
\label{thm:ce}
For any covariance matrix $H$ and mini-batch size $B\ge 1$, $\beta,\eta>0$, the optimal
spectral gap satisfies
\[
\qquad
s^*(H,B)
\ \ge\ 1-8\,\min\!\left\{
\frac{B\,\lambda_{\min}}{\mathrm{tr}(H)}\,,\
\sqrt{\frac{\lambda_{\min}}{\lambda_{\max}}}
\right\}.
\qquad
\]
\end{theorem}

\subsection{Helper Lemmas}

We will refer to any eigenvalue $s$ solving the secular
equation~\eqref{eq:secular-exact} as an \emph{observable eigenvalue},
i.e., an eigenvalue whose eigenmode has nonzero coupling to the scalar
observable $S_w$.
Recall the secular equation for the observable eigenvalue $s$ of
$\mathcal{T}_{H,\eta,\beta,B}$:
\begin{equation}
\label{eq:sec-function}
1
\;=\;
F(s)
\;:=\;
\frac{c}{B}\sum_{i=1}^d
\frac{\lambda_i^2\,\phi_i(s)}{\,1- c\,\lambda_i^2 \phi_i(s)\,}\,,
\qquad
c=\eta^2(1-\beta)^2,
\end{equation}
with
\begin{equation}
\label{eq:phi-def}
\phi_i(s)
=
\frac{s(s+\beta)}{(s-\beta)\,(s-r_{i,+}^2)\,(s-r_{i,-}^2)}\,,
\quad
r_{i,\pm}=\frac{a_i\pm\sqrt{a_i^2-4\beta}}{2},\quad
a_i=(1+\beta)-\eta(1-\beta)\lambda_i.
\end{equation}
The poles of $\phi_i$ are at $s=\beta$ and $s=r_{i,\pm}^2$; for a
deterministically stable HB choice, all these poles lie in $(0,1)$.
We now formalize the intuition that adding stochastic gradient noise can
only \emph{slow} convergence. In particular, it cannot yield a larger
spectral gap than the deterministic Heavy--Ball dynamics.

In the zero-noise case (full batch, $B=\infty$), the Heavy--Ball iteration
on the $[w_t,w_{t-1}]$ state induces the deterministic covariance recursion
\[
S_{t+1} \;=\; A S_t A^\top,
\]
where $A=\mathrm{blkdiag}(A_1,\dots,A_d)$ is the block-diagonal state
matrix in the $H$-basis (with $A_i$ as defined in Equation~\ref{eq:A-i}). Let
\begin{equation}
\label{eq:TB}
\mathcal{T}^{(\infty)}(S) := A S A^\top
\end{equation}
denote this deterministic (zero-noise) covariance operator.

\begin{lemma}[Stochastic spectral radius dominates deterministic]
\label{lem:det_upper_bound}
For every $(\eta,\beta,B)$ we have
\[
\rho\big(\mathcal{T}_{H,\eta,\beta,B}\big)
\ \ge\
\rho\big(\mathcal{T}^{(\infty)}\big)
\;=\; \rho(A)^2.
\]
\end{lemma}

Equivalently, the stochastic spectral gap is \emph{no larger} than the
deterministic gap:
\[
1 - s(H,\eta,\beta,B) \;\le\; 1 - \rho(A)^2.
\]

Before proving this, we record a Perron--Frobenius–type fact for the
covariance operators, which follows from the Krein--Rutman theorem
(Perron--Frobenius for positive operators on cones).

\begin{corollary}[Perron--Frobenius for covariance maps]
\label{lem:PF}
Let $\mathcal{S}^{2d}$ denote the space of real symmetric $2d\times 2d$
matrices and let
$ \mathcal{K} \;:=\; \{\, S \in \mathcal{S}^{2d} : S \succeq 0 \,\} $
be the cone of positive semidefinite (PSD) matrices.
Consider either of the covariance operators
$\mathcal{T} \in \big\{\,\mathcal{T}_{H,\eta,\beta,B},\ \mathcal{T}^{(\infty)}\,\big\}$.
Then:
\begin{itemize}
\item $\mathcal{T}$ is a \emph{positive} operator on the cone $\mathcal{K}$, i.e.
if $
S \succeq 0$, then $\mathcal{T}(S) \succeq 0$.
\item The spectral radius
\[
\rho(\mathcal{T})
:= \max\{\,|s| : s\ \text{eigenvalue of}\ \mathcal{T}\,\}
\]
is attained by a real, nonnegative eigenvalue. That is, there exists
$S_\star \in \mathcal{K}\setminus\{0\}$ and a real $\lambda_\star \ge 0$
such that
\[
\mathcal{T}(S_\star) = \lambda_\star S_\star,
\qquad \lambda_\star = \rho(\mathcal{T}).
\]
\end{itemize}
In particular, each covariance operator admits a principal eigenpair
$(\lambda_\star,S_\star)$ with $\lambda_\star = \rho(\mathcal{T})$ and
$S_\star \succeq 0$.
\end{corollary}

The corollary shows that the leading asymptotic mode of the covariance
dynamics can always be represented by a PSD eigenmatrix.

We now prove Lemma~\ref{lem:det_upper_bound}.

\begin{proof}[Proof of Lemma~\ref{lem:det_upper_bound}]
For brevity, write $\mathcal{T}^{(B)} := \mathcal{T}_{H,\eta,\beta,B}$.
From the exact recursion and the definition
of $\mathcal{T}^{(\infty)}$ in~\eqref{eq:TB}, we can write
\[
\mathcal{T}^{(B)}(S)
= \mathcal{T}^{(\infty)}(S) + \mathcal{N}_B(S),
\]
where the linear ``noise map'' $\mathcal{N}_B$ is
\[
\mathcal{N}_B(S)
:= \eta^2(1-\beta)^2
\begin{bmatrix}
\tfrac{1}{B}\,\Lambda\,\mathrm{Tr}(\Lambda\Sigma_{11}) & 0\\[3pt]
0 & 0
\end{bmatrix},
\qquad
S =
\begin{bmatrix}
\Sigma_{11} & \Sigma_{12}\\[2pt]
\Sigma_{12}^\top & \Sigma_{22}
\end{bmatrix}.
\]
If $S\succeq 0$ then $\Sigma_{11}\succeq 0$, and both
$\Lambda\,\Sigma_{11}\,\Lambda$ and $\Lambda\,\mathrm{tr}(\Lambda\Sigma_{11})$
are PSD. Hence $\mathcal{N}_B(S)\succeq 0$, and therefore
\[
\mathcal{T}^{(\infty)}(S)
\;\preceq\;
\mathcal{T}^{(B)}(S)
\qquad\text{for all }S\succeq 0.
\]

Apply Corollary~\ref{lem:PF} to the deterministic operator
$\mathcal{T}^{(\infty)}$ acting on $\mathcal{K}$. There exists an
eigenpair $(\lambda_\infty,S_\infty)$ with
\[
S_\infty \succeq 0,\quad S_\infty\neq 0,
\qquad
\mathcal{T}^{(\infty)}(S_\infty) = \lambda_\infty S_\infty,
\qquad
\lambda_\infty = \rho\big(\mathcal{T}^{(\infty)}\big).
\]

Since $S_\infty\succeq 0$ and $\mathcal{N}_B$ is PSD-valued on
$\mathcal{K}$, the domination relation implies
\[
\mathcal{T}^{(B)}(S_\infty)
= \mathcal{T}^{(\infty)}(S_\infty) + \mathcal{N}_B(S_\infty)
\succeq \lambda_\infty S_\infty.
\]
By induction,
\[
\big(\mathcal{T}^{(B)}\big)^n(S_\infty)
\succeq \lambda_\infty^n S_\infty
\qquad\text{for all }n\ge 1.
\]

Fix a matrix norm $\|\cdot\|$ on $\mathcal{S}^{2d}$ that is monotone on
the PSD cone: $0\preceq Y\preceq X$ implies $\|Y\|\le \|X\|$. The spectral
norm $\|\cdot\|_2$ has this property, since
$\|X\|_2=\lambda_{\max}(X)$ for $X\succeq 0$ and $Y\preceq X$ implies
$\lambda_{\max}(Y)\le\lambda_{\max}(X)$.

The induced operator norm of $(\mathcal{T}^{(B)})^n$ is
\[
\big\|\big(\mathcal{T}^{(B)}\big)^n\big\|
:= \sup_{S\neq 0} \frac{\big\|\big(\mathcal{T}^{(B)}\big)^n(S)\big\|}{\|S\|}.
\]
Evaluating this supremum at $S=S_\infty$ and using monotonicity on the
cone,
\[
\big\|\big(\mathcal{T}^{(B)}\big)^n\big\|
\;\ge\;
\frac{\big\|\big(\mathcal{T}^{(B)}\big)^n(S_\infty)\big\|}{\|S_\infty\|}
\;\ge\;
\frac{\big\|\lambda_\infty^n S_\infty\big\|}{\|S_\infty\|}
= \lambda_\infty^n.
\]

By Gelfand’s formula,
\[
\rho\big(\mathcal{T}^{(B)}\big)
= \lim_{n\to\infty} \big\|\big(\mathcal{T}^{(B)}\big)^n\big\|^{1/n}
\;\ge\;
\lim_{n\to\infty} \lambda_\infty
= \lambda_\infty
= \rho\big(\mathcal{T}^{(\infty)}\big).
\]

Finally, $\mathcal{T}^{(\infty)}$ acts as $A\otimes A$ on the
vectorized state, so its eigenvalues are products $\lambda_i\lambda_j$
of eigenvalues of $A$. Therefore
\[
\rho\big(\mathcal{T}^{(\infty)}\big)
= \max_{i,j} |\lambda_i\lambda_j|
= \big(\max_i |\lambda_i|\big)^2
= \rho(A)^2.
\]
Combining the two inequalities yields the claimed bound.
\end{proof}

\subsection{Proof}

We start with a step size bound for HB--SGD.

\begin{lemma}[Stability step-size cap]
\label{lem:step-cap}
Assume HB--SGD with parameters $(\eta,\beta,B)$ is stable, i.e., all
eigenvalues of $\mathcal{T}_{H,\eta,\beta,B}$ satisfy $|s|\le 1$.
Then the step size must satisfy
\[
\eta \;\le\; \min\!\left\{\frac{2B}{\mathrm{tr}(H)}\,,\ \frac{2(1+\beta)}{(1-\beta)\lambda_{\max}}\right\}.
\]
\end{lemma}

\begin{proof}
First, let us show that stability forces
\[
F(1)\ \le\ 1.
\]
Any real $s$ with $F(s)=1$ corresponds to an
eigenvalue $s$ of $\mathcal{T}_{H,\eta,\beta,B}$.
In the stable regime, all eigenvalues of $\mathcal{T}_{H,\eta,\beta,B}$
satisfy $|s|\le 1$, so in particular there can be no real eigenvalue
$s>1$. 
$\phi_i(s)=O(1/s)$ implies $F(s)\to 0$ as $s\to\infty$. Therefore, if $F(1)>1$, by the Intermediate Value Theorem, there exists $s>1$ with $F(s)=1$, which would correspond to an eigenvalue $s>1$ and hence contradict
stability. This proves the claim.

At $s=1$, we can evaluate $\phi_i(1)$ explicitly. Using
\[
(1-r_{i,+}^2)(1-r_{i,-}^2)
=(1+\beta)^2-a_i^2
=2(1+\beta)\,\eta(1-\beta)\lambda_i-\eta^2(1-\beta)^2\lambda_i^2,
\]
we obtain
\[
\phi_i(1)
=\frac{1+\beta}{(1-\beta)\,(1-r_{i,+}^2)(1-r_{i,-}^2)}
=\frac{1+\beta}{(1-\beta)}\cdot
\frac{1}{\eta(1-\beta)\lambda_i}\cdot
\frac{1}{\,2(1+\beta)-\eta(1-\beta)\lambda_i\,}.
\]

In the stable regime, each denominator $1-c\lambda_i^2\phi_i(1)$ is
positive (otherwise $F$ would blow up at or before $s=1$ and there
would be an eigenvalue with $s\ge 1$). Using
$x/(1-x)\ge x$ for $x\in[0,1)$, we have
\[
1 \ \ge\ F(1)
= \frac{c}{B}\sum_{i=1}^d \frac{\lambda_i^2\phi_i(1)}{1- c\lambda_i^2\phi_i(1)}
\ \ge\ \frac{c}{B}\sum_{i=1}^d \lambda_i^2\phi_i(1).
\]
Substituting $c=\eta^2(1-\beta)^2$ and the expression for $\phi_i(1)$
gives
\begin{equation}
\label{eq:edge-ineq-simplified}
1 \ \ge\ \frac{\eta(1+\beta)}{B}
\sum_{i=1}^d
\frac{\lambda_i}{\,2(1+\beta)-\eta(1-\beta)\lambda_i\,}.
\end{equation}

We first derive the bound $\eta \le 2B/\mathrm{tr}(H)$.
Dropping the negative term in each denominator,
$2(1+\beta)-\eta(1-\beta)\lambda_i \le 2(1+\beta)$, we obtain
\[
1
\ \ge\ \frac{\eta(1+\beta)}{B}
\sum_{i=1}^d
\frac{\lambda_i}{\,2(1+\beta)-\eta(1-\beta)\lambda_i\,}
\ \ge\ \frac{\eta(1+\beta)}{B}
\sum_{i=1}^d \frac{\lambda_i}{2(1+\beta)}
= \frac{\eta\,\mathrm{tr}(H)}{2B},
\]
which completes the proof of this case.

We now show the second bound of
$\eta \le 2(1+\beta)/((1-\beta)\lambda_{\max})$ must hold.
Lemma~\ref{lem:det_upper_bound} implies that if the
deterministic HB iteration is unstable, i.e. $\rho(A)>1$, then
$\rho(\mathcal{T}_{H,\eta,\beta,B})>1$. Now if $\eta >
2(1+\beta)/((1-\beta)\lambda_{\max})$, then Lemma~\ref{lem:HB-det-band}
implies $\rho(A)>1$, contradicting our
assumed stability. This proves our second claim.
\end{proof}

\begin{lemma}[Gap in terms of $\eta$ and $\beta$]
\label{lem:gap-etabeta}
For any stable HB--SGD parameters $(\eta,\beta,B)$, the spectral gap satisfies
\[
  1 - s(H,\eta,\beta,B)
  \;\le\;
  \min\big\{ 4 \eta \lambda_{\min} , 1-\beta \big\}.
\]
\end{lemma}

\begin{proof}
By Lemma~\ref{lem:det_upper_bound}, the stochastic covariance operator
$\mathcal{T}_{H,\eta,\beta,B}$ has spectral radius at least that of the
deterministic (full-batch) HB operator:
\[
s(H,\eta,\beta,B)
\;\ge\;
\rho\big(\mathcal{T}^{(\infty)}\big)
=
\rho(A)^2,
\]
where $A=\mathrm{blkdiag}(A_1,\dots,A_d)$ is the deterministic HB state
matrix in the $[w_t,w_{t-1}]$ state. Hence
\[
1 - s(H,\eta,\beta,B)
\ \le\
1 - \rho(A)^2.
\]

We first bound the gap by $1-\beta$. For each coordinate $i$, the
per–coordinate block $A_i$ has eigenvalues $r_{i,\pm}$ satisfying
$r_{i,+}r_{i,-}=\beta$. Thus
\[
\max\{|r_{i,+}|,|r_{i,-}|\}^2
\ \ge\ |r_{i,+}r_{i,-}| = \beta.
\]
Taking a maximum over $i$ yields $\rho(A)^2\ge\beta$, so
\[
1 - \rho(A)^2
\ \le\ 1 - \beta.
\]
Therefore
\[
1 - s(H,\eta,\beta,B)
\ \le\ 1 - \beta.
\]

We now show $1-s(H,\eta,\beta,B)\le 4\eta\lambda_{\min}$.
Let $A_{\min}$ be the $2\times2$ HB block corresponding to the smallest
eigenvalue $\lambda_{\min}$, with eigenvalues $r_{\min,\pm}$ solving
\[
z^2 - a_{\min} z + \beta = 0,
\qquad
a_{\min}=(1+\beta)-\eta(1-\beta)\lambda_{\min}.
\]
As $A_{\min}$ is a block of $A$, we have
\[
\rho(A)
\ \ge\ \max\{|r_{\min,+}|,|r_{\min,-}|\}.
\]
Hence
\[
1 - s(H,\eta,\beta,B)
\ \le\ 1 - \rho(A)^2
\ \le\ 1 - \max\{|r_{\min,+}|,|r_{\min,-}|\}^2.
\]

We distinguish the real and complex cases for $r_{\min,\pm}$.

For case of real roots (the overdamped case),
Assume $r_{\min,\pm}\in\mathbb{R}$, which corresponds to $a_{\min}^2>4\beta$.
We have $1 - \rho(A)^2 \leq (1-r_{\min,+}^2) \le 2(1-r_{\min,+})$.
The roots satisfy
\[
(r_{\min,+}-1)(r_{\min,-}-1)
= 1 - (r_{\min,+}+r_{\min,-}) + r_{\min,+}r_{\min,-}
= 1 - a_{\min} + \beta
= \eta(1-\beta)\lambda_{\min}.
\]
Since the roots are real and satisfy $r_+ r_- = \beta$ with $r_- \le
r_+$, we must have $r_- \le \sqrt{\beta}$. Using that $\beta\leq 1$ (stability),
\[
1 - \rho(A)^2 \leq 2(1-r_{\min,+}) = 2\frac{\eta(1-\beta)\lambda_{\min}}{1-r_-}
\le 2\frac{\eta(1-\beta)\lambda_{\min}}{1-\sqrt{\beta}}
= 2\eta(1+\sqrt{\beta})\lambda_{\min}
\leq 4\eta\lambda_{\min},
\]
which completes the proof of this case.

For case of complex roots (underdamped case),
assume $r_{\min,\pm}$ are complex conjugates, where $a_{\min}^2<4\beta$.
Then both have modulus
$\sqrt{\beta}$, so
\[
1 - \rho(A)^2
\ \le\ 1 - |r_{\min,+}|^2
= 1 - \beta.
\]
For the underdamped case (complex roots), due to that
$a_{\min}=(1+\beta)-\eta(1-\beta)\lambda_{\min}$, the condition
$a_{\min}^2<4\beta$ is equivalent to:
\[
  \eta(1-\beta)\lambda_{\min}  > (1-\sqrt{\beta})^2
\quad  \implies \quad
\eta(1+\sqrt{\beta})\lambda_{\min}  > 1-\sqrt{\beta}.
\]
Therefore,
\[
1-\beta \ =\ (1-\sqrt{\beta})(1+\sqrt{\beta})
\leq (1+\sqrt{\beta})^2\eta\lambda_{\min} \ \leq 4\eta\lambda_{\min},
\]
which completes the proof of the complex case.
\end{proof}

Now we are equipped to complete the proof of Theorem~\ref{thm:ce}.

\begin{proof}[Proof of Theorem~\ref{thm:ce}]
Fix $H$ and any stable HB--SGD parameters $(\eta,\beta,B)$.

Combining Lemma~\ref{lem:gap-etabeta} (the step size cap) with
Lemma~\ref{lem:step-cap} (the bound on $1-s(H,\eta,\beta,B)$), we
obtain three simultaneous upper bounds on the same quantity $1-s(H,\eta,\beta,B)$:
\[
1 - s(H,\eta,\beta,B)
\ \le\
\min\!\left\{
4\eta\lambda_{\min},\ 1-\beta
\right\}
\ \le\
\min\!\left\{
\frac{8B\lambda_{\min}}{\mathrm{tr}(H)}\,,\
\frac{8(1+\beta)\lambda_{\min}}{(1-\beta)\lambda_{\max}}\,,\
1-\beta
\right\}.
\]

Using the definition of $s^*(H,B) $, it remains to bound:
\[
1 - s^*(H,B)
\ \le\
 \sup_{\beta\in[0,1)} \ 
\min\!\left\{
\frac{8B\lambda_{\min}}{\mathrm{tr}(H)}\,,\
\frac{8(1+\beta)\lambda_{\min}}{(1-\beta)\lambda_{\max}}\,,\
1-\beta
\right\}.
\]  
The first term inside the minimum does not depend on $\beta)$, so the
proof consists in bounding:
\[
\sup_{\beta\in[0,1)} \ 
\min\!\left\{ \
\frac{8(1+\beta)\lambda_{\min}}{(1-\beta)\lambda_{\max}}\,,\
1-\beta
\right\}.
\]

The first function in the min is increasing in $\beta\in[0,1)$, while
the second is decreasing. Consequently, the $\sup$ is achived at the
at the crossing point $\beta^\star$, where
\[
\frac{8(1+\beta^\star)\lambda_{\min}}{(1-\beta^\star)\lambda_{\max}}
= 1-\beta^\star
\quad\Longleftrightarrow\quad
(1-\beta^\star)^2
= 8(1+\beta^\star)\,\frac{\lambda_{\min}}{\lambda_{\max}}.
\]
Hence
\[
  \sup_{\beta\in[0,1)} \ 
\min\!\left\{ \
\frac{8(1+\beta)\lambda_{\min}}{(1-\beta)\lambda_{\max}}\,,\
1-\beta
\right\} =
1-\beta^\star
= \sqrt{8(1+\beta^\star)\,\frac{\lambda_{\min}}{\lambda_{\max}}}
= 4\,\sqrt{\frac{\lambda_{\min}}{\lambda_{\max}}},
\]
using $1+\beta^\star\le 2$. This completes the proof.
\end{proof}

\begin{corollary}[HB on power-law spectra]
Assume the eigenvalues of $H$ satisfy $\lambda_i \eqsim i^{-a}$ for some $a>1$. Then, the optimal spectral radius of HB satisfies:
\[
s^*(H,B)
\;\gtrsim\;
1 - \min\!\left\{
B d^{-a},\ d^{-a/2}
\right\},
\]
where $\gtrsim$ absorbs universal constants.
In particular, the transition to the accelerated regime occurs at batch size $
B_{\mathrm{HB}}^{\mathrm{crit}}
\eqsim
d^{a/2}$
\end{corollary}
\begin{proof}
Under the power-law spectrum
\[
\lambda_i \eqsim i^{-a},
\qquad
\lambda_{\max}\eqsim 1,
\qquad
\lambda_{\min}\eqsim d^{-a},
\qquad a>1,
\]
Theorem~\ref{thm:ce} gives
\[
s^*(H,B)\;\ge\;1-8\min\!\left\{\frac{B\lambda_{\min}}{\operatorname{tr}(H)},\sqrt{\frac{\lambda_{\min}}{\lambda_{\max}}}\right\}.
\]
Now,
\[
\frac{B\lambda_{\min}}{\operatorname{tr}(H)}
\;\eqsim\;
\frac{B\,d^{-a}}{\sum_{i=1}^d i^{-a}}.
\]
Since \(a>1\), we have
\[
\sum_{i=1}^d i^{-a}\eqsim 1 \implies
\frac{B\lambda_{\min}}{\operatorname{tr}(H)}\eqsim B\,d^{-a}.
\]
Also,
\[
\sqrt{\frac{\lambda_{\min}}{\lambda_{\max}}}
\;\eqsim\;
\sqrt{\frac{d^{-a}}{1}}
=
d^{-a/2}.
\]
Therefore
\[
\min\!\left\{\frac{B\lambda_{\min}}{\operatorname{tr}(H)},\sqrt{\frac{\lambda_{\min}}{\lambda_{\max}}}\right\}
\eqsim
\min\!\left\{B\,d^{-a},\,d^{-a/2}\right\},
\]
\[
\implies s^*(H,B)\;\gtrsim\;1-\min\!\left\{B\,d^{-a},\,d^{-a/2}\right\}.
\]
\end{proof}

\section{Accelerated SGD Analysis}
\label{sec:appendix_asgd_proofs}
We now turn our attention to Accelerated SGD (ASGD) algorithm and we establish rates for it in a similar technical way. We begin by establishing the rates for the deterministic operator $\mathcal{T}_\infty$, followed by deterministic and stochastic conditions for the learning rate. 

Recall that:
\begin{equation}
S_{t+1}
=
A S_t A^\top
+
\frac{\eta^2}{B}
\begin{bmatrix}
(\zeta + (1-\beta))^2 \,\Tr(\Lambda S^{11}_{t})\, \Lambda
&
\zeta(\zeta + 1-\beta)\,\Tr(\Lambda S^{11}_{t})\, \Lambda
\\[4pt]
\zeta(\zeta + 1-\beta)\,\Tr(\Lambda S^{11}_{t})\, \Lambda
&
\zeta^2 \,\Tr(\Lambda S^{11}_{t})\, \Lambda
\end{bmatrix}.
\end{equation}

Let $\mathcal{T}_{H,\eta,\beta,\zeta,B}$ denote the corresponding linear covariance
update operator on symmetric $2d\times 2d$ matrices, defined by
\begin{equation}
\label{eq:TS_zeta}
S_{t+1} = \mathcal{T}_{H,\eta,\beta, \zeta,B}(S_t)
\end{equation}

Denoting $\textbf{p} = [(\zeta + (1-\beta))^2, \zeta(\zeta + 1 - \beta),  \zeta(\zeta + 1 - \beta), \zeta^2]$. 

Let $\gamma = \mathrm{Tr}(H S_t^{11})$ and \(A=\mathrm{blkdiag}(A_1,\ldots,A_d)\), where each per–coordinate \(2\times2\) block is
\begin{equation}
A_i=
\begin{bmatrix}
(1+\beta) - \eta(\zeta +1 - \beta)\lambda_i  & -\beta\\[2pt]
1 - \eta\zeta\lambda_i & 0
\end{bmatrix} \notag
\end{equation}

Pushing a $\vec$ through Equation~\ref{eqn:st_recur_zeta} coupled with the fact that $\vec(AS_tA^\top) = (A \otimes A)\vec(S_t)$ gives us:

\begin{align}
    s \vec(S_t) &= (A \otimes A) \vec(S_t) + \frac{\eta^2 \gamma}{B} (\textbf{p} \otimes \vec(\Lambda))\ \notag
\end{align}
\begin{align}
        \implies \qty(sI - A \otimes A) \vec(S_t) &= \frac{ \eta^2 \gamma}{B} (\textbf{p} \otimes \vec(\Lambda)) \notag
\end{align}

We can break the above equation up into $d$, $2 \times 2$ equations each using a block of $A$:

\begin{align*}
    (sI - A_i \otimes A_i) \vec(S_{t,i}) = \frac{\eta^2 \gamma\lambda_i }{B} \textbf{p}
\end{align*}
\begin{align}
    \implies \vec(S_{t,i}) &= \frac{\eta^2 \gamma\lambda_i }{B} (sI - A_i \otimes A_i)^{-1} \textbf{p}
    \label{eqn:recur_Si_p}
\end{align}

Let,
\[
q := \zeta + 1 - \beta,
\qquad
a := (1+\beta) - \eta(\zeta+1-\beta)\lambda_i,
\qquad
b := -\beta,
\qquad
c := 1 - \eta \zeta \lambda_i,
\quad M(s)_i := (sI - A_i \otimes A_i)^{-1}.
\]
Let $\{e_1, e_2, \dots, e_d\}$ represent the standard basis vector. To simplify $\gamma$, we multiply Equation~\ref{eqn:recur_Si_p} by $\lambda_ie_1$ and sum over $i \in \{1, \dots, d\}$ to get:

\begin{equation}
\label{eqn:sec_eq_asgd}
1 = F(s) := \frac{\eta^2}{B} \sum_i \lambda_i^2 e_1^{\top} M(s)_i\, \textbf{p}
\end{equation}
with 
\[
e_1^{\top} M(s)_i\, \textbf{p}
=
\frac{(s-bc)\left(s q^{2} + b^{2}\zeta^{2}\right)
      + 2ab\, s\, \zeta\, q}
     {(s+bc)\left((s-bc)^{2} - a^{2}s\right)}
\]

\[
=\frac{
(s + \beta)\bigl(s(\zeta + 1 - \beta)^2 + \beta^2 \zeta^2\bigr)
- 2\beta(1 + \beta)s\zeta(\zeta + 1 - \beta)
+ \beta \eta \zeta \lambda_i \bigl(s(\zeta + 1 - \beta)^2 - \beta^2 \zeta^2\bigr)
}{
(s - \beta + \beta \eta \zeta \lambda_i)
\Bigl[
(s - 1)(s - \beta^2)
+ 2\eta \lambda_i \bigl(s(1 + \zeta - \beta^2) - \beta^2 \zeta\bigr)
+ \eta^2 \lambda_i^2 \bigl(\beta^2 \zeta^2 - s(\zeta + 1 - \beta)^2\bigr)
\Bigr]
}
\]

As argued for the Heavy-Ball case, one can show that the maximum eigenvalue of $\mathcal{T}_{H,\eta,\beta,\zeta,B}$ will occur at a symmetric PSD matrix, thus, justifying the vectorization of Equation~\ref{eqn:st_recur_zeta}.

\paragraph{Bounding $\T_\infty$.}
By the same argument as in the Heavy--Ball case, the spectral radius of the
stochastic covariance operator dominates that of the deterministic operator:
\[
\rho(\mathcal{T}_{H,\eta,\beta,\zeta,B})\ge \rho(\T_\infty),
\qquad
\T_\infty(M)=AMA^\top.
\]
Since \(A=\operatorname{blkdiag}(A_1,\dots,A_d)\), restricting to the space of
symmetric matrices gives
\[
\rho(\T_\infty)=\rho(A\otimes A)=\rho(A)^2=\max_i \rho(A_i)^2.
\]
Thus it suffices to understand the eigenvalues of each block
\[
A_i=
\begin{bmatrix}
(1+\beta)-\eta(\zeta+1-\beta)\lambda_i & -\beta\\[2pt]
1-\eta\zeta\lambda_i & 0
\end{bmatrix}.
\]

We now record two lower bounds on \(\rho(A)^2\).

\paragraph{First bound (product bound / complex-root regime).}
Let \(r_{i,+},r_{i,-}\) be the two roots of the characteristic polynomial of
\(A_i\). Their product is
\[
r_{i,+}r_{i,-}=\beta(1-\eta\zeta\lambda_i).
\]
Hence
\[
\rho(A_i)^2
=
\max\{|r_{i,+}|,|r_{i,-}|\}^2
\ge |r_{i,+}r_{i,-}|
=
\beta(1-\eta\zeta\lambda_i).
\]
Taking \(i\) corresponding to \(\lambda_{\min}\), we obtain
\[
\rho(A)^2\ge \beta(1-\eta\zeta\lambda_{\min}),
\]
and therefore
\begin{equation}
\label{eq:asgd-det-prod-hbstyle}
1-\rho(A)^2
\le
(1-\beta)+\eta\beta\zeta\lambda_{\min}.
\end{equation}

\paragraph{Second bound (real-root regime).}
Let \(A_{\min}\) denote the block corresponding to \(\lambda_{\min}\), and let
\(r_+\ge r_-\) be its two real roots. Then
\[
\rho(A)^2\ge \rho(A_{\min})^2\ge r_+^2,
\]
so
\[
1-\rho(A)^2\le 1-r_+^2\le 2(1-r_+).
\]
From the characteristic polynomial at \(\lambda_{\min}\),
\[
(1-r_+)(1-r_-)=\eta\lambda_{\min}(1-\beta)(1+\zeta).
\]
Also, since \(r_-+r_+=a_{\min}\le 1+\beta\), we have
\[
2r_-\le 1+\beta
\qquad\Longrightarrow\qquad
1-r_-\ge \frac{1-\beta}{2}.
\]
Therefore
\[
1-r_+
=
\frac{\eta\lambda_{\min}(1-\beta)(1+\zeta)}{1-r_-}
\le
2\eta\lambda_{\min}(1+\zeta),
\]
and hence
\begin{equation}
\label{eq:asgd-det-real-hbstyle}
1-\rho(A)^2\le 4\eta\lambda_{\min}(1+\zeta).
\end{equation}

Combining \eqref{eq:asgd-det-prod-hbstyle} and
\eqref{eq:asgd-det-real-hbstyle}, we obtain
\begin{equation}
\label{eq:asgd-gap-etabeta-zeta-hbstyle}
1-\rho(A)^2
\lesssim
\min\Bigl\{
\eta\lambda_{\min}(1+\zeta),
\,
(1-\beta)+\eta\beta\zeta\lambda_{\min}
\Bigr\}.
\end{equation}

\subsection{Deterministic learning-rate stability}

\begin{lemma}[Deterministic ASGD stability band]
\label{lem:ASGD-det-band-hbstyle}
Assume \(\eta>0\). The deterministic ASGD iteration
\[
S_{t+1}=AS_tA^\top
\]
is stable, i.e. all eigenvalues of every \(A_i\) lie in the open unit disk, if
and only if
\[
0<\eta<
\frac{2(1+\beta)}
{\bigl(1-\beta+\zeta(1+\beta)\bigr)\lambda_{\max}}.
\]
\end{lemma}

\begin{proof}
For each coordinate \(i\), write
\[
A_i=
\begin{bmatrix}
a_i & -\beta\\
c_i & 0
\end{bmatrix},
\qquad
a_i=(1+\beta)-\eta(\zeta+1-\beta)\lambda_i,
\qquad
c_i=1-\eta\zeta\lambda_i.
\]
Its characteristic polynomial is
\[
r^2-a_i r+\beta c_i=0.
\]
Applying the degree--\(2\) Jury criterion to
\[
r^2+pr+q=0
\qquad\text{with}\qquad
p=-a_i,\quad q=\beta c_i,
\]
the roots lie in the open unit disk if and only if
\[
|q|<1,\qquad 1+p+q>0,\qquad 1-p+q>0.
\]

The condition \(1+p+q>0\) becomes
\[
1-a_i+\beta c_i>0,
\]
which simplifies to
\[
\eta\lambda_i(1-\beta)(1+\zeta)>0.
\]
This holds automatically for \(\eta>0\), \(\lambda_i>0\), \(0<\beta<1\), and
\(\zeta>0\).

The condition \(1-p+q>0\) becomes
\[
1+a_i+\beta c_i>0,
\]
that is,
\[
2(1+\beta)-\eta\lambda_i\Bigl((\zeta+1-\beta)+\beta\zeta\Bigr)>0.
\]
Since
\[
(\zeta+1-\beta)+\beta\zeta
=
1-\beta+\zeta(1+\beta),
\]
this is equivalent to
\[
\eta<
\frac{2(1+\beta)}
{\bigl(1-\beta+\zeta(1+\beta)\bigr)\lambda_i}.
\]
Imposing this for every \(i\) yields the claimed condition at
\(\lambda_{\max}\).
\end{proof}

\subsection{A stochastic step-size cap}

Let
\[
s(H,\eta,\beta,\zeta,B)
:=
\rho\big(\mathcal{T}_{H,\eta,\beta,\zeta,B}\big),
\qquad
\alpha(H,\eta,\beta,\zeta,B)
:=
1-s(H,\eta,\beta,\zeta,B).
\]

Recall from the secular equation \eqref{eqn:sec_eq_asgd} that any real solution
of \(F(s)=1\) is an eigenvalue \(s\) of \(\mathcal T_{H,\eta,\beta,\zeta,B}\).
Since \(F(s)\to 0\) as \(s\to\infty\), stability implies
\[
F(1)\le 1.
\]

Evaluating the secular equation at \(s=1\) gives
\begin{equation}
\label{eq:asgd-F1-exact-hbstyle}
F(1)
=
\frac{\eta}{B}
\sum_{i=1}^d
\lambda_i
\frac{
(1-\beta^2)(1+\zeta)
+
\beta\eta\zeta\lambda_i(\zeta+1-\beta+\beta\zeta)
}{
\bigl((1-\beta)+\beta\eta\zeta\lambda_i\bigr)
\bigl(2(1+\beta)-\eta\lambda_i(1-\beta+\zeta(1+\beta))\bigr)
}.
\end{equation}

Define
\begin{equation}
S:=\{i:\beta\eta\zeta\lambda_i\le 1-\beta\},
\qquad
S^c:=\{i:\beta\eta\zeta\lambda_i>1-\beta\}.
\label{eq:define_split}
\end{equation}

Let
\[
x_i := 1-\beta(1-\eta\zeta\lambda_i) = (1-\beta)+\beta\eta\zeta\lambda_i,
\qquad
N_i := (1-\beta^2)(1+\zeta)+\beta\eta\zeta\lambda_i(\zeta+1-\beta+\beta\zeta).
\]
Following the definition of $x_i$ and $N_i$, Equation~\ref{eq:asgd-F1-exact-hbstyle} can be written as, 
\[
F(1)
=
\frac{\eta}{B}\sum_i \lambda_i
\frac{N_i}{x_i\left(2(1+\beta)-\eta\lambda_i\bigl(1-\beta+\zeta(1+\beta)\bigr)\right)},
\]
Since,
\[
2(1+\beta)-\eta\lambda_i\bigl(1-\beta+\zeta(1+\beta)\bigr)\le 2(1+\beta),
\]
we have,
\[
\frac{1}{x_i\left(2(1+\beta)-\eta\lambda_i\bigl(1-\beta+\zeta(1+\beta)\bigr)\right)}
\ge
\frac{1}{2(1+\beta)x_i}.
\]
\[
\implies F(1)
\ge
\frac{\eta}{2B(1+\beta)}
\sum_i \lambda_i \frac{N_i}{x_i}.
\]
\begin{align*}
&=1 \frac{\eta}{2B(1+\beta)} \sum \lambda_i \frac{(1-\beta^2)(1+\zeta) + \beta \eta \zeta \lambda_i(\zeta + 1 - \beta + \beta\zeta)}{\bigl(1 - \beta(1-\eta\zeta\lambda_i)\bigr)}
\end{align*}

Splitting in two cases:

\textbf{Small eigenvalues: $\beta\eta\zeta\lambda_i \leq 1-\beta$}
\[
x_i=(1-\beta)+\beta\eta\zeta\lambda_i \le 2(1-\beta) \implies
\frac{1}{x_i}\ge \frac{1}{2(1-\beta)}.
\]
Hence
\[
F(1)
\ge
\frac{\eta}{4B(1+\beta)}
\sum_{i\in S}\lambda_i \frac{N_i}{1-\beta}.
\]
Now
\[
N_i
=
(1-\beta^2)(1+\zeta)+\beta\eta\zeta\lambda_i(\zeta+1-\beta+\beta\zeta)
\ge
(1-\beta^2)(1+\zeta),
\]
\[
\implies \frac{N_i}{1-\beta}
\ge
\frac{(1-\beta^2)(1+\zeta)}{1-\beta}
=
(1+\beta)(1+\zeta).
\]
Substituting back we get,
\begin{equation}
F(1)
\ge
\frac{\eta(1+\zeta)}{4B}\sum_{i\in S}\lambda_i.
\label{eqn:small_eig}
\end{equation}

\textbf{Higher eigenvalues: $\beta\eta\zeta\lambda_i \geq 1-\beta$}
\[
x_i=(1-\beta)+\beta\eta\zeta\lambda_i \le 2\beta\eta\zeta\lambda_i \implies
\frac{1}{x_i}\ge \frac{1}{2\beta\eta\zeta\lambda_i}.
\]
Therefore
\[
F(1)
\ge
\frac{\eta}{4B(1+\beta)}
\sum_{i\in S^c}\lambda_i \frac{N_i}{\beta\eta\zeta\lambda_i}.
\]
\[
\implies F(1)
\ge
\frac{1}{4B(1+\beta)\beta\zeta}
\sum_{i\in S^c}
\left[(1-\beta^2)(1+\zeta)+\beta\eta\zeta\lambda_i(\zeta+1-\beta+\beta\zeta)\right].
\]
\[
=
\frac{(1-\beta^2)(1+\zeta)}{4B(1+\beta)\beta\zeta}|S^c|
+
\frac{\eta(\zeta+1-\beta+\beta\zeta)}{4B(1+\beta)}
\sum_{i\in S^c}\lambda_i.
\]
\begin{equation}
\implies F(1)
\ge
\frac{(1-\beta)(1+\zeta)}{4B\beta\zeta}|S^c|
+
\frac{\eta(\zeta+1-\beta+\beta\zeta)}{4B(1+\beta)}
\sum_{i\in S^c}\lambda_i.
\label{eqn:large_eig}
\end{equation}

Summing up Equation~\ref{eqn:small_eig},~\ref{eqn:large_eig} we get:
\[
F(1)
\ge
\frac{\eta(1+\zeta)}{4B}\sum_{i\in S}\lambda_i
+
\frac{(1-\beta)(1+\zeta)}{4B\beta\zeta}|S^c|
+
\frac{\eta(\zeta+1-\beta+\beta\zeta)}{4B(1+\beta)}
\sum_{i\in S^c}\lambda_i.
\]
Since,
\[
\frac{\zeta+1-\beta+\beta\zeta}{1+\beta}
=
\zeta+\frac{1-\beta}{1+\beta},
\]
we have,
\[
F(1)
\ge
\frac{\eta(1+\zeta)}{4B}\sum_{i\in S}\lambda_i
+
\frac{\eta\zeta}{4B}\sum_{i\in S^c}\lambda_i
+
\frac{\eta(1-\beta)}{4B(1+\beta)}\sum_{i\in S^c}\lambda_i
+
\frac{(1-\beta)(1+\zeta)}{4B\beta\zeta}|S^c|.
\]
Since
\[
\frac{\eta(1+\zeta)}{4B}\sum_{i\in S}\lambda_i
=
\frac{\eta}{4B}\sum_{i\in S}\lambda_i
+
\frac{\eta\zeta}{4B}\sum_{i\in S}\lambda_i,
\]
we get
\begin{equation}
F(1)
\ge
\frac{\eta}{4B}\sum_{i\in S}\lambda_i
+
\frac{\eta\zeta}{4B}\Tr(H)
+
\frac{\eta(1-\beta)}{4B(1+\beta)}\sum_{i\in S^c}\lambda_i
+
\frac{(1-\beta)(1+\zeta)}{4B\beta\zeta}|S^c|.
\label{eq:asgd-F1-lower-hbstyle}
\end{equation}

We now assume the power-law spectrum
\[
\lambda_i\eqsim i^{-a},
\qquad
\lambda_{\max}\eqsim 1,
\qquad
\lambda_{\min}\eqsim d^{-a},
\qquad
a>1.
\]
Let \(k:=|S^c|\). Since the spectrum is monotone, \(S^c\) consists of the top
\(k\) eigenvalues up to constants, so
\[
\sum_{i\in S}\lambda_i
\eqsim
\sum_{i=k+1}^d i^{-a},
\qquad
\Tr(H)\eqsim 1.
\]

For $a>1$ we have:
\[
\sum_{i=k+1}^d i^{-a}\eqsim k^{1-a}.
\]
Since all terms in \eqref{eq:asgd-F1-lower-hbstyle} are nonnegative, we may
drop the last two terms and obtain
\[
F(1)
\ge
\frac{\eta}{4B}\sum_{i\in S}\lambda_i
+
\frac{\eta\zeta}{4B}\Tr(H).
\]
Using the power-law estimates above, this gives
\[
F(1)
\gtrsim
\frac{\eta}{B}k^{1-a}
+
\frac{\eta\zeta}{B}.
\]
Since stability implies \(F(1)\le 1\), we obtain
\[
1\gtrsim \frac{\eta}{B}k^{1-a}+\frac{\eta\zeta}{B},
\]
hence the stochastic step-size cap
\begin{equation}
\label{eq:asgd-step-cap-stoch-hbstyle}
\eta\lesssim \frac{B}{k^{1-a}+\zeta}.
\end{equation}

\subsection{Proof of the ASGD rate bound}

We define the optimal ASGD spectral radius and gap at batch size \(B\) by
\[
s^*(H,B)
:=
\inf_{\eta>0,\ \beta\in[0,1),\ \zeta>0}
s(H,\eta,\beta,\zeta,B),
\qquad
\alpha^*(H,B)
:=
1-s^*(H,B).
\]

For every stable choice of \((\eta,\beta,\zeta)\), the deterministic dominance
argument gives
\[
\alpha(H,\eta,\beta,\zeta,B)
\le
1-\rho(A)^2.
\]
Combining this with \eqref{eq:asgd-gap-etabeta-zeta-hbstyle}, we obtain
\begin{equation}
\label{eq:asgd-gap-basic-hbstyle}
\alpha(H,\eta,\beta,\zeta,B)
\lesssim
\min\Bigl\{
\eta\lambda_{\min}(1+\zeta),
\,
(1-\beta)+\eta\beta\zeta\lambda_{\min}
\Bigr\}.
\end{equation}

We split into two cases.

\paragraph{Case 1: \(\zeta\ge 1\).}
In this regime, using \eqref{eq:asgd-gap-basic-hbstyle} and
\eqref{eq:asgd-step-cap-stoch-hbstyle},
\[
\alpha(H,\eta,\beta,\zeta,B)
\lesssim
\eta\lambda_{\min}(1+\zeta)
\lesssim
\frac{1+\zeta}{\zeta}\lambda_{\min}
\eqsim
d^{-a}.
\]
Thus the branch \(\zeta\ge 1\) is at best of order \(d^{-a}\), so it cannot
yield acceleration.

\paragraph{Case 2: \(0<\zeta<1\).}
In this regime, \(1+\zeta\eqsim 1\), so the first term in
\eqref{eq:asgd-gap-basic-hbstyle} is
\[
\eta\lambda_{\min}(1+\zeta)\eqsim \eta d^{-a}.
\]
For the second term, using \eqref{eq:define_split},
\[
(1-\beta)+\eta\beta\zeta\lambda_{\min}
\le
(1-\beta)+\eta\zeta d^{-a}
\lesssim
\eta\zeta k^{-a}+\eta\zeta d^{-a}
\lesssim
\eta\zeta k^{-a},
\]
since \(k\le d\) implies \(k^{-a}\ge d^{-a}\). Therefore
\begin{equation}
\label{eq:asgd-gap-small-zeta-hbstyle}
\alpha(H,\eta,\beta,\zeta,B)
\lesssim
\eta\,\min\{d^{-a},\,\zeta k^{-a}\}.
\end{equation}

Now combine \eqref{eq:asgd-gap-small-zeta-hbstyle} with the two step-size caps
\eqref{eq:asgd-step-cap-stoch-hbstyle} and
Lemma~\ref{lem:ASGD-det-band-hbstyle}. For every stable choice of
\((\eta,\beta,\zeta)\) with \(0<\zeta<1\),
\begin{equation}
\label{eq:asgd-four-term-bound-hbstyle}
\alpha(H,\eta,\beta,\zeta,B)
\lesssim
\min\left\{
\frac{B\,d^{-a}}{k^{1-a}+\zeta},
\,
\frac{d^{-a}}{\zeta},
\,
\frac{B\,\zeta\,k^{-a}}{k^{1-a}+\zeta},
\,
k^{-a}
\right\}.
\end{equation}

At this point, \(\eta\) and \(\beta\) have been eliminated: the right-hand side
is an upper bound valid for every stable \((\eta,\beta,\zeta)\), and it depends
only on the induced threshold \(k\) and on \(\zeta\). Thus
\begin{equation}
\label{eq:asgd-gap-sup-k-zeta-hbstyle}
\alpha^*(H,B)
\lesssim
\sup_{1\le k\le d,\ 0<\zeta<1}
\min\left\{
\frac{B\,d^{-a}}{k^{1-a}+\zeta},
\,
\frac{d^{-a}}{\zeta},
\,
\frac{B\,\zeta\,k^{-a}}{k^{1-a}+\zeta},
\,
k^{-a}
\right\}.
\end{equation}

We now optimize over \(\zeta\) for fixed \(k\). Set
\[
\zeta_k:=\left(\frac{k}{d}\right)^a.
\]
This is the crossover point where
\[
d^{-a}=\zeta k^{-a}.
\]

For fixed \(k\), define
\[
A_k(\zeta):=\frac{B\,d^{-a}}{k^{1-a}+\zeta},
\qquad
B_k(\zeta):=\frac{d^{-a}}{\zeta},
\qquad
C_k(\zeta):=\frac{B\,\zeta\,k^{-a}}{k^{1-a}+\zeta},
\qquad
D_k:=k^{-a}.
\]
Then \(A_k\) and \(B_k\) are decreasing in \(\zeta\), \(C_k\) is increasing in
\(\zeta\), and \(D_k\) is constant.

If \(0<\zeta\le \zeta_k\), then \(\zeta k^{-a}\le d^{-a}\), so
\[
A_k(\zeta)\ge C_k(\zeta),
\qquad
B_k(\zeta)\ge D_k.
\]
Hence in this region
\[
\min\{A_k(\zeta),B_k(\zeta),C_k(\zeta),D_k\}
=
\min\{C_k(\zeta),D_k\},
\]
which is nondecreasing in \(\zeta\).

If \(\zeta_k\le \zeta<1\), then \(d^{-a}\le \zeta k^{-a}\), so
\[
A_k(\zeta)\le C_k(\zeta),
\qquad
B_k(\zeta)\le D_k.
\]
Hence in this region
\[
\min\{A_k(\zeta),B_k(\zeta),C_k(\zeta),D_k\}
=
\min\{A_k(\zeta),B_k(\zeta)\},
\]
which is nonincreasing in \(\zeta\).

Therefore, for each fixed \(k\), the right-hand side of
\eqref{eq:asgd-gap-sup-k-zeta-hbstyle} is maximized at the crossover point
\(\zeta=\zeta_k\). Substituting \(\zeta_k=(k/d)^a\), we obtain
\begin{equation}
\label{eq:asgd-gap-sup-k-only-hbstyle}
\alpha^*(H,B)
\lesssim
\sup_{1\le k\le d}
\min\left\{
\frac{B\,d^{-a}}{k^{1-a}+\left(\frac{k}{d}\right)^a},
\,
k^{-a}
\right\}.
\end{equation}

We now optimize over \(k\).

\paragraph{Regime I: \(B\lesssim 1\).}
Let
\[
f(k):=k^{1-a}+\left(\frac{k}{d}\right)^a.
\]
Its minimum is obtained by balancing the two terms:
\[
k^{1-a}\eqsim \left(\frac{k}{d}\right)^a
\qquad\Longrightarrow\qquad
k\eqsim d^{\frac{a}{2a-1}}.
\]
At this value,
\[
f(k)\eqsim d^{-\frac{a(a-1)}{2a-1}},
\]
and therefore
\[
\frac{B\,d^{-a}}{f(k)}
\eqsim
B\,d^{-\frac{a^2}{2a-1}}.
\]
Also,
\[
k^{-a}\eqsim d^{-\frac{a^2}{2a-1}}.
\]
Hence, when \(B\lesssim 1\), the first term is the smaller one, and we obtain
\[
\alpha^*(H,B)
\lesssim
B\,d^{-\frac{a^2}{2a-1}}.
\]

\paragraph{Regime II: \(1\lesssim B\lesssim d^{1/2}\).}
In this regime, the optimum is obtained by balancing the two terms in
\eqref{eq:asgd-gap-sup-k-only-hbstyle}:
\[
\frac{B\,d^{-a}}{k^{1-a}+\left(\frac{k}{d}\right)^a}
\eqsim
k^{-a}.
\]
Equivalently,
\[
B
\eqsim
d^a k^{-a}\left(k^{1-a}+\left(\frac{k}{d}\right)^a\right)
=
d^a k^{1-2a}+1.
\]
For \(B\gtrsim 1\), this gives
\[
k\eqsim \left(\frac{d^a}{B}\right)^{\frac1{2a-1}}.
\]
Substituting back,
\[
\alpha^*(H,B)
\lesssim
k^{-a}
\eqsim
B^{\frac{a}{2a-1}}\,d^{-\frac{a^2}{2a-1}}.
\]

\paragraph{Regime III: deterministic ceiling.}
Finally, stochastic noise can only slow convergence, so
\[
\alpha^*(H,B)\le \alpha^*_{\det}(H),
\]
where \(\alpha^*_{\det}(H)\) is the optimal full-batch ASGD gap. The standard
deterministic ASGD optimization gives
\[
\alpha^*_{\det}(H)\lesssim d^{-a/2}.
\]
Therefore
\[
\alpha^*(H,B)\lesssim d^{-a/2}
\qquad\text{for all }B.
\]

The crossover expression
\[
B^{\frac{a}{2a-1}}\,d^{-\frac{a^2}{2a-1}}
\]
matches the deterministic ceiling \(d^{-a/2}\) exactly when \(B\eqsim d^{1/2}\).
Combining the three regimes, we obtain
\begin{equation}
\label{eq:asgd-final-rate-hbstyle}
\alpha^*(H,B)
\;\lesssim\;
\begin{cases}
B\, d^{-\frac{a^2}{2a-1}}, & B \lesssim 1,\\[6pt]
B^{\frac{a}{2a-1}}\, d^{-\frac{a^2}{2a-1}}, & 1 \lesssim B \lesssim d^{1/2},\\[6pt]
d^{-a/2}, & B \gtrsim d^{1/2}.
\end{cases}
\end{equation}

\newpage

\end{document}